\documentclass[journal]{IEEEtai}

\usepackage[colorlinks,urlcolor=blue,linkcolor=blue,citecolor=blue]{hyperref}

\usepackage{color,array}

\usepackage{amsmath,amsfonts,amssymb}
\usepackage{amsthm}
\usepackage{algorithmic}
\usepackage{algorithm}
\usepackage{subfigure}
\usepackage{textcomp}
\usepackage{stfloats}
\usepackage{url}
\usepackage{verbatim}
\usepackage{graphicx}
\usepackage{cite}
\hypersetup{hidelinks}
\usepackage{tabularx}
\usepackage{booktabs} 
\usepackage{multirow}
\usepackage{multicol}
\usepackage{xcolor}

\newtheorem{theorem}{Theorem}

\usepackage{adjustbox}
\usepackage{tikz}
\usetikzlibrary{intersections}
\usetikzlibrary{shapes,arrows,positioning}
\usetikzlibrary{calc,arrows.meta}
\usepackage{xparse}
\NewDocumentCommand\Cycle{O{} m m m O{} m}{
  \draw[#1](#2.{#3+asin(#6/(#4*1.41))}) arc (180+#3-45:180+#3-45-270:#6/2) #5;
} 
\tikzstyle{block} = [rectangle, draw]
\tikzstyle{input} = [coordinate]
\tikzstyle{output} = [coordinate]
\tikzstyle{pinstyle} = [pin edge={to-,thin,black}]

\begin{document}

\title{TrumorGPT: Graph-Based Retrieval-Augmented Large Language Model for Fact-Checking}

\author{Ching~Nam~Hang, Pei-Duo~Yu and~Chee~Wei~Tan
\thanks{This research was supported in part by the Saint Francis University Institute Development Grant (Publication) (IDG-P240205), the Singapore Ministry of Education Academic Research Fund Tier 1 (RG91/22) and Tier 2 (T2EP20224-0031), as well as the Research Grants Council of Hong Kong under its Institutional Development Scheme Research Infrastructure Grant (No. UGC/IDS(R)11/21) and Collaborative Research Grant (No. UGC/IDS(C)11/E01/24).}
\thanks{The material in this paper was presented in part at the 2024 58th Annual Conference on Information Sciences and Systems, Princeton, USA.}
\thanks{C. N. Hang is with the Yam Pak Charitable Foundation School of Computing and Information Sciences, Saint Francis University, Hong Kong (e-mail: cnhang@sfu.edu.hk).}
\thanks{P.-D. Yu is with the Department of Applied Mathematics, Chung Yuan Christian University, Taiwan (e-mail: peiduoyu@cycu.edu.tw).}
\thanks{C. W. Tan is with the College of Computing and Data Science, Nanyang Technological University, Singapore (e-mail: cheewei.tan@ntu.edu.sg).}
}




\maketitle

\begin{abstract}
In the age of social media, the rapid spread of misinformation and rumors has led to the emergence of infodemics, where false information poses a significant threat to society. To combat this issue, we introduce \textit{TrumorGPT}, a novel generative artificial intelligence solution designed for fact-checking in the health domain. TrumorGPT aims to distinguish ``trumors'', which are health-related rumors that turn out to be true, providing a crucial tool in differentiating between mere speculation and verified facts. This framework leverages a large language model (LLM) with few-shot learning for semantic health knowledge graph construction and semantic reasoning. TrumorGPT incorporates graph-based retrieval-augmented generation (GraphRAG) to address the hallucination issue common in LLMs and the limitations of static training data. GraphRAG involves accessing and utilizing information from regularly updated semantic health knowledge graphs that consist of the latest medical news and health information, ensuring that fact-checking by TrumorGPT is based on the most recent data. Evaluating with extensive healthcare datasets, TrumorGPT demonstrates superior performance in fact-checking for public health claims. Its ability to effectively conduct fact-checking across various platforms marks a critical step forward in the fight against health-related misinformation, enhancing trust and accuracy in the digital information age.
\end{abstract}

\begin{IEEEImpStatement}
Fact-checking for public health claims is essential in the current digital landscape, where misinformation can quickly spread and influence public health decisions. Unverified or misleading health information can lead to harmful behaviors, mistrust in medical recommendations, and the undermining of public health campaigns. For example, during the COVID-19 pandemic, false claims about vaccines caused significant confusion, reducing vaccination rates and contributing to the continued spread of the virus. Accurate and timely verification of such claims is critical to ensure that individuals make informed decisions based on factual information. TrumorGPT addresses this urgent need by providing a robust, AI-powered fact-checking framework for public health claims. Through its use of GraphRAG with semantic health knowledge graphs, it verifies claims with up-to-date information, helping users differentiate between fact and misinformation. By doing so, TrumorGPT plays a crucial role in ensuring trust in public health communications and safeguarding the well-being of society.
\end{IEEEImpStatement}


\section{Introduction}
\label{sec:introduction}
\IEEEPARstart{I}{n} the recent digital environment, the rise of infodemics, characterized by the widespread and rapid dissemination of misinformation through social media, has emerged as a significant issue \cite{acemoglu2023model, hang2023mega, siderius2021misinformation, acemoglu2016network}. The Coronavirus Disease 2019 (COVID-19) Infodemic declared by the World Health Organization (WHO) in 2020 is the first infodemic of a global scale, relating to the spread of misinformation related to the COVID-19 pandemic. Epidemics and rumors share similar traits \cite{acemoglu2011opinion, del2016spreading, vosoughi2018spread, tan2023contagion}. Waves of infodemics always follow after a pandemic occurs, and they spread faster, literally at the speed of light, in the cyberspace of today. Whether it is through hoaxes or viral conspiracy theories, information travels fast these days due to the extensive online social networks that have a huge number of users and are playing increasingly important and critical roles in our society for information retrieval \cite{luo2021spread}. This ``word-of-mouth'' mechanism can cause rumors and false information to spread rampantly, and, oftentimes, may have serious repercussions for individuals and nations \cite{acemoglu2024simple}. Misinformation and disinformation circulating in cyberspace go viral on the web and online social networks in a short time, eventually leading to widespread confusion, panic, and harmful actions among the public. For instance, false claims about medication efficacy can lead to adverse health effects from misuse. The complexities of the false information spreading phenomenon and its nature as a large-scale problem rapidly outstrip our ability to keep pace with rigorous fact-checking efforts.

The rise of artificial intelligence (AI) technologies in content generation has added a new dimension to creating and disseminating misleading information. AI tools have simplified the process of generating and distributing content that closely resembles authentic news, blurring the distinction between real reporting and fabricated narratives. This is not just a theoretical concern; the advent of generative AI has dramatically reshaped the landscape of misinformation in 2023. For example, NewsGuard, an organization dedicated to monitoring misinformation, noted a dramatic rise in AI-driven disinformation, from geopolitical propaganda to healthcare myths. In particular, 614 unreliable AI-generated news and information websites in 15 languages were identified, demonstrating the ability of generative AI to complicate the separation of fact from fiction. The spread of misleading health information online significantly impedes the search for credible medical advice and trustworthy sources, making it hard for people to know what to believe. Thus, there is an urgent need for reliable fact-checking, especially during critical crises like global pandemics.

Fact-checking is the process of verifying information to determine its accuracy and truthfulness \cite{walter2020fact}. Traditionally, this task has been manually performed by journalists and researchers who cross-reference claims with credible sources. However, the sheer volume and speed of information generated in the digital age have made manual fact-checking increasingly challenging \cite{alhindi2018your}. The sheer volume of information circulating online, making it difficult for human fact-checkers to keep pace with the constant influx of new content. Additionally, the rapid spread of misinformation through social media platforms often outpaces fact-checking efforts, leading to the widespread dissemination of false or misleading information before it can be adequately addressed. The process is further complicated by the subjective nature of fact-checking, as biases may inadvertently influence the evaluation of information \cite{vaughan2018making}. Moreover, fact-checking often focuses on explicit claims, leaving subtleties like framing and context manipulation unaddressed. This situation is exacerbated by the lack of collaborative efforts and standardized fact-checking methodologies in the current society \cite{dafoe2021cooperative}.

In response to these challenges, advancements in large language models (LLMs) \cite{touvron2023llama, ouyang2022training, brown2020language, anil2023palm, openai2023gpt4} and the development of knowledge graphs have emerged as important solutions in automating the fact-checking process. LLMs, trained on extensive datasets, excel in understanding and generating text, making them adept at parsing and assessing the veracity of statements. Meanwhile, knowledge graphs provide a structured representation of relationships between entities and facts, offering a rich database for cross-referencing claims. However, LLMs can produce information that is coherent but factually incorrect. For instance, LLMs might generate a response based on outdated health information, leading to inaccuracies in current health advice. This issue often stems from the dependence of models on their training data, which may not always be up-to-date or fully comprehensive. Knowledge graphs, while invaluable in providing structured factual data, face challenges in coverage and timeliness. Keeping these graphs updated with the latest information is an ongoing challenge, especially in rapidly evolving fields like public health and epidemiology. Therefore, knowledge graphs might not cover all areas of expertise equally, potentially leading to gaps in available information for fact-checking. It is critical to address these limitations inherent in LLMs and knowledge graphs and explore how they can complement and enhance each other to mitigate the shortcomings and optimize the overall fact-checking process.

In this study, we present \textit{TrumorGPT}, a novel generative AI framework designed for fact-checking to combat health-related misinformation. TrumorGPT, aptly named, focuses on identifying ``trumors'', a term that blends ``true'' and ``rumor'' to describe rumors that ultimately prove to be factual. The concept of a trumor encapsulates instances where what initially appears as mere gossip or unfounded claims eventually aligns with reality. Thus, TrumorGPT serves to uncover the truth from health-related rumors, effectively bridging the gap between skepticism and fact in public health discourse. A unique aspect of TrumorGPT is its integration of an LLM with semantic health knowledge graphs for fact verification through semantic reasoning. TrumorGPT employs three key technologies for verifying the truthfulness of health news: topic-enhanced sentence centrality \cite{zheng2019sentence} and topic-specific TextRank \cite{mihalcea2004textrank, ceri2013introduction, kazemi2020biased, florescu2017positionrank} coupled with few-shot learning to enhance the construction of semantic knowledge graphs, and Generative Pre-trained Transformer 4 (GPT-4) \cite{openai2023gpt4} integrated with graph-based retrieval-augmented generation (GraphRAG) \cite{lewis2020retrieval, edge2024local} to provide up-to-date health knowledge in semantic reasoning. The joint use of GraphRAG and natural language processing (NLP) techniques of LLMs in TrumorGPT offers a comprehensive approach to fact-checking, making it a powerful tool in the fight against health-related misinformation. Thus, TrumorGPT provides a promising solution for fact-checking, effectively countering the spread of misleading health information.

Overall, the contributions of this study are as follows:
\begin{itemize}
    \item We introduce TrumorGPT, a novel generative AI framework that jointly integrates an LLM with GraphRAG, utilizing semantic health knowledge graphs for fact-checking to fight against health-related misinformation.
    \item We propose the topic-specific TextRank algorithm with topic-enhanced sentence centrality to enhance the construction of semantic health knowledge graphs, thereby improving the semantic analysis of the LLM.
    \item We collect data from the latest verified health news to create semantic health knowledge graphs for GraphRAG, providing updated and accurate health knowledge to optimize the performance of the LLM.
    \item We demonstrate the superior performance of TrumorGPT by conducting extensive evaluations on verifying the accuracy of medical knowledge and the truthfulness of health-related news.
\end{itemize}

This paper is organized as follows. In Section \ref{sec:related}, we review the related work on using language models for providing medical and health information and fact-checking. In Section \ref{sec:framework}, we introduce TrumorGPT, a framework that leverages an LLM enhanced with GraphRAG to verify health-related facts through semantic reasoning. We detail the construction of semantic health knowledge graphs using topic-enhanced sentence centrality and topic-specific TextRank as well as the integration of these graphs with GraphRAG for improved accuracy in fact verification. In Section \ref{sec:exp}, we present the performance evaluation results, showcasing the effectiveness of TrumorGPT in accurately fact-checking health-related queries. We discuss and analyze various RAG approaches in Section \ref{sec:dis}. Finally, we conclude the paper in Section \ref{sec:conclusion}.

\section{Related Work}
\label{sec:related}
Utilizing language models to address medical and health problems has become a popular approach. For example, the authors in \cite{rasmy2021med} develop Med-BERT, a novel adaptation of the bidirectional encoder representations from transformers (BERT) framework for the structured electronic health records domain, demonstrating that its pretraining on a large dataset significantly enhances prediction accuracy in clinical tasks. The work in \cite{de2014medical} explores a variation of neural language models that learn from concepts derived from structured ontologies and free-text, employing it to measure semantic similarity between medical concepts, thereby enhancing techniques in medical informatics and information retrieval. In \cite{yang2023one}, the authors develop the LLM-Synergy framework, an innovative ensemble learning pipeline utilizing state-of-the-art LLMs to significantly enhance accuracy and reliability in diverse medical question-answering tasks through two novel ensemble methods. The work in \cite{ye2023qilin} introduces a multi-stage training method with a Chinese Medicine dataset to significantly enhance the performance of LLMs in healthcare applications. Other studies assess the effectiveness of LLMs in tackling medical and health-related issues \cite{gilbert2023large, tang2023evaluating, mesko2023impact}. These models are designed to function as healthcare assistants, providing advice within this domain only. While they might deliver factual healthcare information, they are not equipped to verify the accuracy of specific health-related news or information.

In response to the infodemic, WHO launched a platform to combat COVID-19 misinformation through targeted communication strategies and partnerships with social media platforms and international agencies \cite{zarocostas2020fight}. The authors in \cite{liu2021framework} propose an AI-based framework to enhance eHealth literacy and combat infodemics by improving access to accurate health information. The work in \cite{Gallotti_2020} develops an Infodemic Risk Index by analyzing Twitter messages, enabling early detection of unreliable information to combat the COVID-19 infodemic with targeted communication strategies. MEGA in \cite{hang2023mega} further enhances the risk index by leveraging machine learning and graph analytics to detect spambots and identify influential spreaders, thus refining the approach to infodemic risk management. These approaches act as passive surveillance tools, merely tracking misinformation instead of proactively fighting against health-related rumors.

Addressing the infodemic requires the deployment of proactive fact-checking mechanisms to effectively identify and mitigate misinformation. Traditional fact-checking, carried out by expert journalists and analysts, finds it challenging to keep up with the overwhelming amount of information constantly produced in the digital world. Automated fact-checking, integrating NLP and machine learning, becomes particularly important for accurately verifying claims in the modern rapid information flow \cite{guo2022survey}. For instance, the authors in \cite{karadzhov2017fully} propose a fully-automatic fact-checking framework utilizing a deep neural network with LSTM text encoding, which leverages the web as a knowledge source and combines semantic kernels with task-specific embeddings, for rumor detection and fact verification within community question answering forums. The work in \cite{zhou2019physiological} introduces a fact-checking framework for machine learning, using visualization of training data influence and physiological signals to measure and enhance user trust in predictive models. FactChecker in \cite{nakashole2014language} applies a language-aware truth-finding approach, utilizing linguistic analysis to evaluate source objectivity and trustworthiness, and combines these factors with co-mention influence for a more accurate believability assessment of facts. Using deep learning, the authors in \cite{harrag2022arabic} develop a neural network model specifically for identifying Arabic fake news, leveraging convolutional neural networks and a balanced Arabic corpus to enhance the accuracy of fact-checking.

Knowledge graphs play a critical role in fact-checking, providing a structured and interconnected database of facts that significantly improves the accuracy and efficiency of verifying information. In \cite{ciampaglia2015computational}, the authors demonstrate that computational fact-checking through knowledge graphs can efficiently validate information, effectively differentiating true from false claims using Wikipedia data. Tracy in \cite{gad2019tracy} enhances knowledge graph curation by providing clear explanations for fact verification through rule-based semantic traces and a user-friendly interface. The work in \cite{shi2016fact} proposes a link-prediction model in knowledge graphs for computational fact-checking, which successfully assesses the truthfulness of various claims using large-scale graphs from Wikipedia and SemMedDB. FACE-KEG in \cite{vedula2021face} enhances explainable fact-checking by constructing and utilizing knowledge graphs to evaluate the veracity of claims and generate understandable explanations. Other works \cite{koloski2022knowledge, hu2021compare, mayank2022deap} leverage knowledge graphs for fake news detection, employing innovative models and methods to improve accuracy in fact-checking across various contexts.

\section{Methods}
\label{sec:framework}
In this section, we introduce TrumorGPT, a framework that utilizes a pretrained LLM for fact-checking. The process initiates with query processing, where input from users undergoes semantic similarity analysis to find relevant health information. It is important to note that the input query from users should pertain to health-related content. Otherwise, TrumorGPT may not be able to accurately process and provide relevant information. TrumorGPT employs few-shot learning, using the topic-enhanced sentence centrality and topic-specific TextRank algorithms with examples to prepare the LLM for the creation of semantic health knowledge graphs. These graphs incorporate information from the latest and updated medical and health knowledge base, facilitating the capability of the pretrained LLM for graph-based retrieval-augmented generation (GraphRAG) in the health domain. The framework retrieves and synthesizes health information, enriching the semantic health knowledge graph, which is then used to apply semantic reasoning to the fact-checking task. The final output is a semantically reasoned answer that determines the factual accuracy of the input query. Fig. \ref{fig:archi_trumorgpt} provides a high-level overview of the TrumorGPT framework.

\begin{figure*}[tbp]
    \centering
    \includegraphics[width=.95\textwidth]{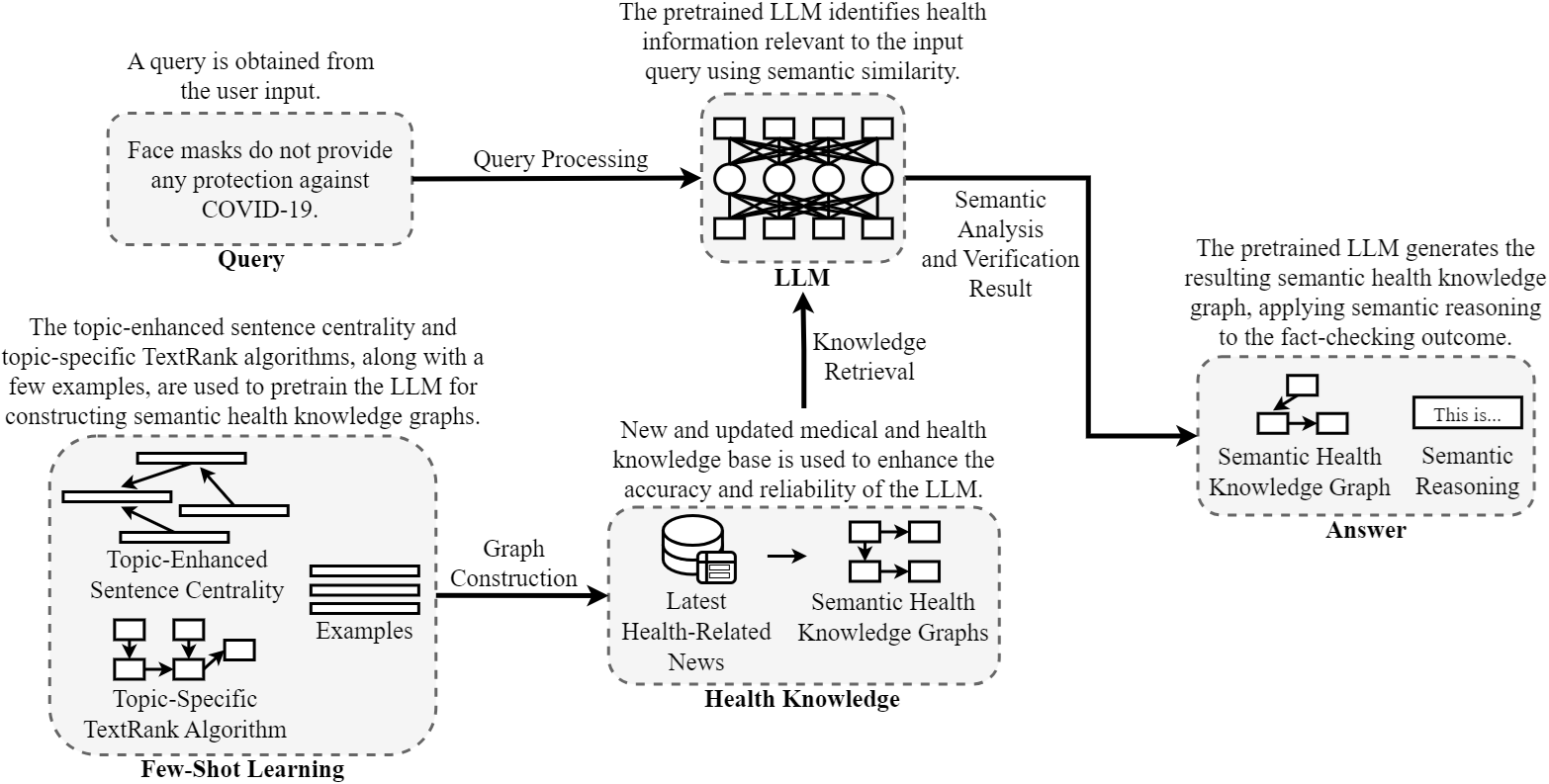}
    \caption{The architecture of TrumorGPT, showcasing the workflow from user input to fact verification. TrumorGPT leverages advanced algorithms and an extensive medical and health knowledge base to generate a verified semantic health knowledge graph and provide a reasoned answer for fact-checking.}
    \label{fig:archi_trumorgpt}
\end{figure*}

\subsection{Semantic Health Knowledge Graph}
\label{sec:knowledge_G}
A semantic health knowledge graph is an effective mechanism for encapsulating health knowledge in a format that is both structured and interpretable by machines. This graph consists of vertices that symbolize entities and edges that represent their connections. The ``semantic'' aspect of the graph ensures that entities and their interrelations are based on meaningful, contextually relevant concepts, making them understandable to both machines and humans.

We represent a semantic health knowledge graph as a directed graph $G = \{E, R, F\}$, where $E$ denotes the set of entities, $R$ indicates the set of relations, and $F$ is the triple set of relational facts. Each fact is represented as a triple $(h, r, t)\in F$, where $h\in E$ is the head entity, $t\in E$ is the tail entity, and $r\in R$ is the relationship between $h$ and $t$. For example, (``$\texttt{SARS-CoV-2}$'', ``$\texttt{causes}$'', ``$\texttt{COVID-19}$'') represents the fact that SARS-CoV-2 causes the COVID-19 pandemic. In a semantic health knowledge graph constructed from a set of statements extracted from a resource like Wikipedia, factual relations among entities in those statements are represented as a large-scale network. The truthfulness of a new statement is evaluated based on its presence as an edge in this graph or the existence of a short path in the graph linking its subject to its object. Otherwise, the absence of such edges or short paths usually indicates that the statement is untrue. As a result, the task of fact-checking involves verifying if a query, represented as a triple $(h, r, t)$, is true within the knowledge graph framework. In Fig. \ref{fig:kg_example}, we present a semantic health knowledge graph illustrating the impacts of COVID-19, vaccine development, and the differing public health policies of zero-COVID and living-with-COVID implemented by China and the United States during the pandemic.

\begin{figure}[tbp]
    \centering
    \includegraphics[width=\linewidth]{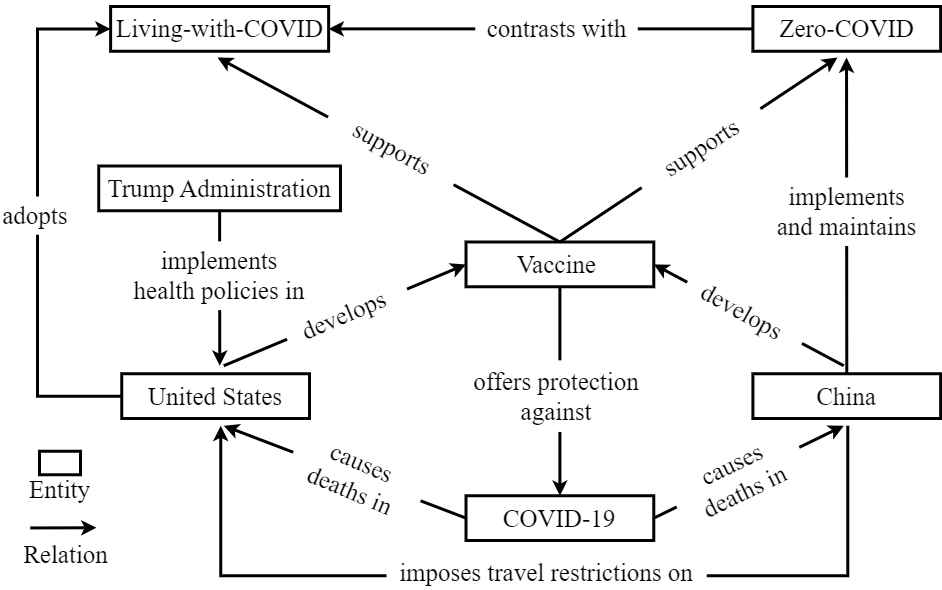}
    \caption{An illustrative semantic health knowledge graph, highlighting the global response to COVID-19 through vaccine efforts and contrasting health strategies adopted by China and the United States.}
    \label{fig:kg_example}
\end{figure}

Creating a health knowledge graph can be challenging due to the vast amount of health information available, particularly when the user input is as extensive as an article. The key lies in identifying the central topics or the main ideas from the abundant data. This process of extracting key health information enables the creation of a semantic health knowledge graph that accurately represents the core themes and relevant connections within the larger text, ensuring that the most significant health information is used for further semantic analysis and inference. To achieve this, we employ a two-step keyword extraction method. First, we use a sentence centrality-based approach to extract important sentences from the article. Then, from these summarized key sentences, we apply a ranking-based method to extract the most relevant health-related keywords. We enhance the original algorithms for sentence centrality \cite{zheng2019sentence} and TextRank \cite{mihalcea2004textrank, ceri2013introduction} by increasing the weights of sentences or words related to health topics, ensuring that sentence centrality and TextRank are more likely to assign higher relevance scores to health-related content. This approach enables us to identify key entities in health topics more effectively.

\subsection{Topic-Enhanced Sentence Centrality}\label{sec:tesc}
Topic-enhanced sentence centrality integrates topic modeling with centrality measures to identify sentences that are both central to the document and specifically relevant to the topic of interest. We combine conventional BERT embeddings with a topic-enhanced latent Dirichlet allocation (LDA) model to compute the topic-enhanced sentence centrality.

\subsubsection{Topic-Enhanced Embeddings with LDA and BERT}
\label{sec:Topic_LDA}
In LDA, topic distributions can be computed for both words and documents, typically representing the distribution of topics within a document. This concept can be extended to sentences by treating each sentence as a small ``document''.  When applying LDA at the sentence level, the topic distribution reflects the probability of a sentence belonging to each of the identified topics. Each sentence is treated as an individual document, and the LDA model is trained on this data. This allows the model to produce a topic distribution for each sentence, just as it does for full documents. Once the model is trained, we can transform any sentence to obtain its corresponding topic distribution vector.

The process begins by collecting a large corpus of documents related to a specific domain, such as health, to train the LDA model with domain-specific focus. This corpus may include health-related articles, research papers, and blog posts. Once trained, the LDA model captures health-related topics and can be applied to new sentences to generate topic distribution vectors, representing the probabilities that each sentence pertains to different health-related topics.

BERT embeddings are used to further enhance the semantic representation of sentences. Each sentence is encoded into a high-dimensional vector using BERT, capturing rich semantic features. These BERT embeddings are then combined with LDA-derived topic distribution vectors to form representations that integrate both semantic and topical information. For each sentence $s$, let $\mathbf{e}_s$ represent its BERT embedding vector and $\mathbf{t}_s=[t_{s,1},t_{s,2},\ldots ,t_{s,K}]$ denote its topic distribution vector from LDA, where $t_{s,k}$ is the probability that sentence $s$ belongs to topic $k$ and $\sum\limits_{k=1}^K t_{s,k}=1$. Rather than a simple concatenation, we introduce a weighting parameter $\eta \in [0,1]$ to balance semantic and topical contributions. Let 
$$
\tilde{e}_s = \frac{e_s}{\lVert e_s \rVert} \quad \text{and} \quad \tilde{t}_s = \frac{t_s}{\lVert t_s \rVert}
$$
denote the normalized BERT embedding and LDA topic distribution vector, respectively. The final topic-enhanced embedding is then defined as
$$
v_s = \left[\eta\tilde{e}_s; \,(1-\eta)\tilde{t}_s\right],
$$
where $[\,\cdot\,;\,\cdot\,]$ denotes vector concatenation. In our experiments, we set $\eta= 0.7$, giving slightly more weight to the semantic information from BERT than to the LDA-derived topics.

\subsubsection{Topic-Enhanced Centrality}
To identify critical sentences within the text, we construct a weighted graph $G_s$ using the combined vectors $\mathbf{v}_s$, where each vertex represents a sentence. Edges between vertices are determined based on the similarity between their corresponding vectors. Specifically, the weight $w^{s}_{i,j}$ of the edge connecting vertex $i$ and vertex $j$ is defined using cosine similarity:
\[ w^s_{i,j}=\frac{\mathbf{v}_i\cdot \mathbf{v}_j}{\lVert \mathbf{v}_i\lVert \lVert \mathbf{v}_j\lVert}.\]

We then apply a centrality measure, such as PageRank, to the graph to identify the most central sentences, which are both semantically meaningful and topically aligned with the health domain. This method ensures that the centrality computation incorporates both the domain-specific knowledge from the LDA model and the deep semantic understanding from BERT, yielding a more refined and insightful text analysis.

\subsection{Topic-Specific TextRank}\label{sec:tst}
Topic-specific TextRank (TST) adapts the traditional TextRank algorithm \cite{mihalcea2004textrank} to prioritize sentences based on their relevance to a specific topic, thereby making it more sensitive to thematic elements. In its original form, TextRank is used for keyword extraction or sentence ranking in summarization and is computed as:
\begin{align*}
    \text{TR}(v_i)=(1-d)+d\cdot \sum\limits_{j\in In(v_i)}\frac{w_{j,i}}{\sum\limits_{k\in Out(v_j)} w_{j,k}}\text{TR}(v_j),
\end{align*}
where $\text{TR}(v_i)$ denotes the TextRank score of vertex $v_i$, $d$ is the damping vector, $In(v_i)$ is the set of vertices that link to $v_i$, $Out(v_j)$ is the set of vertices linked from $v_j$, and $w_{j,i}\geq 0$ is the weight of the edge from $v_j$ to $v_i$.

We adapt the conventional TextRank into TST by incorporating the principles of topic-specific PageRank. The primary modifications involve vertex selection, edge weighting, and the scoring algorithm, all tailored to emphasize topic relevance. We begin by defining a topic relevance score $R(v_i)$ for $v_i$, measuring how closely $v_i$ (a word or sentence) relates to the target topic. Here, we slightly abuse the notation $v_i$ in a different font to represent its graph-theoretic attribute (as a vertex). To emphasize health-related content, we modify the topic relevance score $R(v_i)$ for each sentence (vertex) in the graph. Specifically, if certain topics are identified as health-related, we assign them higher weights. Suppose topics $k_1, k_2, \ldots, k_m$ are associated with health. Then, for sentence $i$, we define:
\begin{equation*}
R(v_i) = \frac{\sum_{k=1}^{K} \beta_k \cdot t_{i,k}}{\sum_{j} \sum_{k=1}^{K} \beta_k \cdot t_{j,k}},
\end{equation*}
where
\begin{equation*}
\beta_k =
\begin{cases}
\alpha & \text{if topic } k \in \{k_1, k_2, \ldots, k_m\}\\[1mm]
1 & \text{otherwise}.
\end{cases}    
\end{equation*}

In our implementation, we set $\alpha = 1.5$, giving health-related topics 50\% more influence. The normalization guarantees that
$\sum_j R(v_j) = 1$. This weighted score $R(v_i)$ is then used in constructing the sentence similarity graph and running the TST algorithm. For simplicity, we denote the vector of topic relevance scores $R(v_i)$ as $\mathbf{u}=[R(v_1), R(v_2), \ldots ,R(v_n)]$. 

To adjust the edge weight $w_{j,i}$ to reflect both the relationship between vertices and their combined relevance to the topic, we define the adjusted weight $w'_{j,i}$ as
\begin{equation*}
w'_{j,i}=\frac{R(v_i)+R(v_j)}{2}\cdot w_{j,i}.
\end{equation*}
Then, we formulate the TST as:
\begin{align}
\label{eq:TST}
    \text{TST}(v_i)=(1-d)\cdot R(v_i)+d\cdot \sum\limits_{j\in In(v_i)}\frac{w'_{j,i}\cdot \text{TST}(v_j)}{\sum\limits_{k\in Out(v_j)} w'_{j,k}}.
\end{align}

We observe that, in contrast to conventional PageRank, TST modifies the uniform teleportation to prioritize topic-relevant vertices. Additionally, the second term in (\ref{eq:TST}) adjusts the weight of the random walk based on the topic relevance score $R(v_i)$. To rewrite (\ref{eq:TST}) in the form of a Markov chain, we first define a stochastic matrix $\mathbf{P}=[p_{i,j}]_{n \times n}$ as
\begin{align}
\label{eq:P}
    p_{i,j}=\frac{w'_{j,i}}{\sum\limits_{k\in Out(v_j)}w'_{j,k}}.
\end{align}

To show that the matrix $\mathbf{P}$ is stochastic, we need to verify that each element $p_{i,j}$ is non-negative and that the sum of each column in $\mathbf{P}$ equals 1. By definition, we can deduce that $w'_{j,i} \geq 0$ since $w'_{j,i}$ is the average of two positive numbers. Next, we show that the sum of each column in $\mathbf{P}$ is $1$. From the definition of $p_{i,j}$ in \eqref{eq:P}, we have
\[
\sum_{i=1}^n p_{i,j} = \sum_{i=1}^n \frac{w'_{j,i}}{\sum\limits_{k \in \text{Out}(v_j)} w'_{j,k}}= \frac{\sum_{i=1}^n w'_{j,i}}{\sum\limits_{k \in \text{Out}(v_j)} w'_{j,k}}=1,
\]
where the results above follow from the fact that
\begin{equation*}
\sum_{i=1}^n w'_{j,i} = \sum_{k \in \text{Out}(v_j)} w'_{j,k}. 
\end{equation*}
Since each element $p_{i,j}$ is non-negative and the sum of each column is $1$, $\mathbf{P}$ is column-stochastic. This ensures that $\mathbf{P}$ can serve as a valid transition matrix for a Markov chain.

We now define the transition matrix $\mathbf{P}'$ for the Markov chain as
\begin{align}
\label{eq:P'}
    \mathbf{P}'=d\cdot \mathbf{P}+(1-d)\cdot \mathbf{E},
\end{align}
where $\mathbf{E}=ue^T$ is a matrix in which each column vector is $u$ and $e$ is the all-ones vector. Thus, each column of $\mathbf{E}$ sums to $1$. Additionally, since $\mathbf{P}'$ is an affine combination of two stochastic matrices, $\mathbf{P}'$ itself is also a stochastic matrix. The power iteration can then be expressed as
\begin{align}
\label{eq:tst_markov}
    \text{TST}^{(t+1)}=\mathbf{P}'\cdot \text{TST}^{(t)}.
\end{align}

\begin{theorem}\label{thm:1}
The Markov chain defined by the transition matrix $\mathbf{P}'$ is irreducible and aperiodic, and therefore converges to a unique stationary distribution with strictly positive entries.  
\end{theorem}

\begin{proof}
For the transition matrix defined in \eqref{eq:tst_markov}, each state can communicate with every other state, as $R(v_i)$ is strictly positive and ensures that $(1-d)\cdot \mathbf{E}$ provides a positive probability of moving from any state to any other in a single step. Thus, the Markov chain is irreducible. In addition, the probability of remaining in the same state (from $v_i$ to $v_i$) is non-zero, guaranteeing that every state has the possibility of returning to itself at each time step. This implies that the period of every state is $1$, thereby establishing aperiodicity. According to the Perron-Frobenius theorem, since $\mathbf{P}'$ is a non-negative, irreducible, and aperiodic matrix, it possesses a unique largest eigenvalue $\lambda_1 =1$, with all other eigenvalues $\lambda_i$ having magnitudes strictly less than $1$. The eigenvector corresponding to $\lambda_1$ is unique and contains strictly positive entries, representing the stationary distribution.
\end{proof}

Theorem \ref{thm:1} ensures the convergence of TST. Let $\mathbf{TST}^*$ represent the unique stationary distribution of \eqref{eq:tst_markov}. Then, we present the following theorem, which characterizes the speed of the convergence of TST.

\begin{theorem}\label{thm:2}
Let $\mathbf{P}$ be defined as in \eqref{eq:P}. Then, the convergence rate of the Markov chain in \eqref{eq:tst_markov} can be characterized using $\mathbf{P}$ as:
\begin{align*}
\| \mathbf{TST}^{(t)} - \mathbf{TST}^* \| < C \cdot d^t \cdot |\lambda_2(\mathbf{P})|^t,
\end{align*} 
where $C$ is a constant determined by the initial distribution.
\end{theorem}

\begin{proof}
The convergence of the TST algorithm is governed by the spectral properties of its transition matrix. Let $\mathbf{P}'$ be the transition matrix defined in \eqref{eq:P'}, where $\mathbf{P}$ is the original column-stochastic matrix, $\mathbf{E}=\mathbf{u}e^T$ is a rank-one stochastic matrix corresponding to topic relevance scores, and $d$ is the damping factor. The eigenvalues of $\mathbf{P}'$ can be characterized using the theorem in eigenvalue perturbation theory \cite{DING20071223}. Specifically, if $\mathbf{P}$ is a column-stochastic matrix with eigenvalues $1, \lambda_2, \ldots, \lambda_n$ and $\mathbf{E}$ is a rank-one stochastic matrix with eigenvalues $1, 0, \ldots, 0$, then the eigenvalues of $\mathbf{P}'$ are given by $1, d \lambda_2, d \lambda_3, \ldots, d \lambda_n$. Therefore, the second-largest eigenvalue of $\mathbf{P}'$ is $\lambda_2(\mathbf{P}') = d \lambda_2(\mathbf{P})$.

The convergence rate of the TST algorithm can be assessed by examining the difference between the iteratively computed TST vector $\mathbf{TST}^{(t)}$ and the stationary distribution \(\mathbf{TST}^*\). From \cite{MC_Levin}, we have:
\begin{align*}
\| \mathbf{TST}^{(t)} - \mathbf{TST}^* \| \leq C \cdot (|\lambda_2(\mathbf{P}')|)^t = C \cdot (d |\lambda_2(\mathbf{P})|)^t,
\end{align*}
where $C$ is a constant that depends on the initial distribution. This result indicates that the convergence rate is exponentially related to the second-largest eigenvalue $\lambda_2(\mathbf{P}')$ of the transition matrix $\mathbf{P}'$. For example, for a typical damping factor $d = 0.85$, we have
\begin{align*}
\| \mathbf{TST}^{(t)} - \mathbf{TST}^* \| < C \cdot 0.85^t \cdot |\lambda_2(\mathbf{P})|^t.
\end{align*}

According to the Perron-Frobenius theorem, the largest eigenvalue of a column-stochastic matrix is $1$. Additionally, based on the graph spectral properties, the second largest eigenvalue of a connected graph is strictly less than $1$. Importantly, when $|\lambda_2(\mathbf{P})|$ is significantly less than $1$, the algorithm converges more rapidly.
\end{proof}

Theorems \ref{thm:1} and \ref{thm:2} highlight the importance of the topic-specific teleportation term $(1-d)\mathbf{E}$ in the transition matrix $\mathbf{P}'$, which ensures that the Markov chain is irreducible and aperiodic, thus guaranteeing the existence of a unique stationary distribution. However, the convergence rate of the TST algorithm is mainly determined by the second-largest eigenvalue of $\mathbf{P}'$. Specifically, this eigenvalue is given by $\lambda_2(\mathbf{P}') = d \lambda_2(\mathbf{P})$, where $\lambda_2(\mathbf{P})$ denotes the second-largest eigenvalue of the original transition matrix $\mathbf{P}$. Thus, while the term $(1-d)\mathbf{E}$ ensures the convergence of the Markov chain, the convergence rate is primarily governed by the magnitude of $\lambda_2(\mathbf{P})$ and the damping factor $d$.

To validate these results, we conduct numerical simulations on a graph where the second-largest eigenvalue of the original transition matrix $\mathbf{P}$ is $\lambda_2(\mathbf{P}) = 0.9899$. We examine the convergence speed of the TST algorithm with various damping factors $d$. Theoretical convergence factors, given by $d \cdot |\lambda_2(\mathbf{P})|$, are calculated and annotated in Fig. \ref{fig:d_vs_speed}.

\begin{figure}
    \centering
    \includegraphics[width=\linewidth]{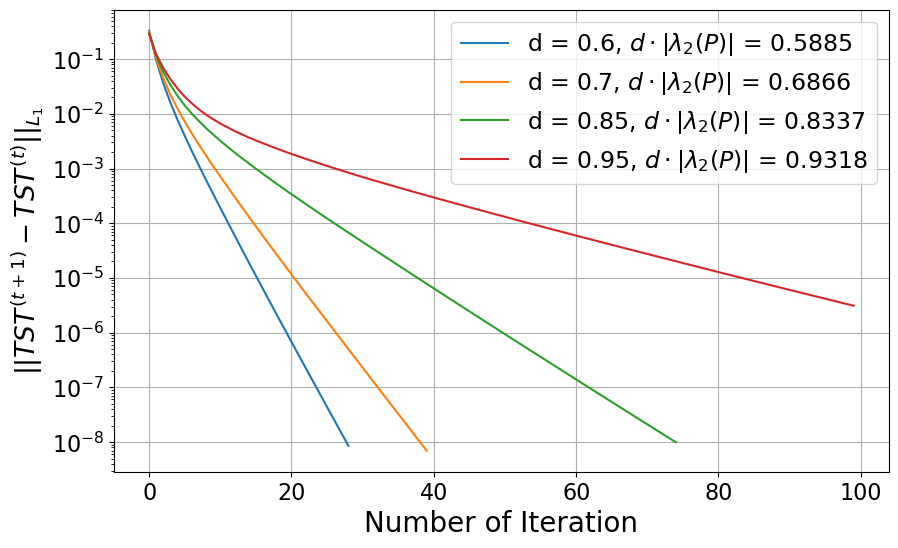}
    \caption{Illustration of the impact of the damping factor and $\lambda_2(\mathbf{P})$ on the convergence speed of (\ref{eq:tst_markov}). The $y$-axis is displayed on a log scale.}
    \label{fig:d_vs_speed}
\end{figure}

The simulation is conducted on a randomly generated graph using the Watts-Strogatz model with parameters $(k,n)=(4,500)$. The numerical results show that with $d = 0.6$, the algorithm converges after $30$ iterations, whereas with $d = 0.95$, it fails to converge within $100$ iterations. These findings confirm that as the damping factor $d$ increases, the convergence factor $d \cdot |\lambda_2(\mathbf{P})|$ approaches 1, leading to a slower convergence rate. This aligns with Theorem \ref{thm:2}, which illustrates that the convergence rate is indeed primarily affected by the magnitude of $\lambda_2(\mathbf{P})$ and the damping factor $d$. Therefore, although the term $(1-d)\mathbf{E}$ ensures the convergence of the Markov chain by making it irreducible and aperiodic, the actual speed of convergence is governed by $d \lambda_2(\mathbf{P})$.

In practice, we set the damping factor to $d = 0.85$, which is the default setting in PageRank. The iterative power method terminates once the change in the TST vector falls below $10^{-6}$ in the $\ell_1$ norm. We set the maximum number of iterations to $500$ to ensure convergence, and the initial TST vector is uniformly initialized, i.e., $\mathbf{TST}^{(0)}(v_i) = \frac{1}{n}$ for $i = 1, \ldots, n$.

\subsection{Large Language Model with Few-Shot Learning}
LLMs have become a fundamental part of NLP, offering impressive performance in tasks like text generation, sentiment analysis, translation, and question-answering. Trained on extensive datasets, these models use deep learning architectures to effectively process and interpret various aspects of language, making them highly useful in diverse NLP applications. In the context of knowledge graphs, LLMs can understand and analyze substantial amounts of text to identify relevant connections and information necessary for creating detailed and accurate knowledge graphs. In TrumorGPT, we particularly use Generative Pre-trained Transformer 4 (GPT-4) \cite{openai2023gpt4} to facilitate the fact-checking process.

To enable TrumorGPT to build accurate semantic health knowledge graphs, we instruct GPT-4 to leverage the algorithms for computing topic-enhanced sentence centrality and topic-specific TextRank proposed in Sections \ref{sec:tesc} and \ref{sec:tst}. However, GPT-4 may struggle with mathematical concepts, making it challenging to comprehend and accurately implement algorithms involving complex math, which can lead to errors in the generated knowledge graphs. To mitigate this issue, we leverage the Advanced Data Analysis (ADA) capabilities of GPT-4 by instructing GPT-4 to use Python to implement the proposed algorithms for constructing knowledge graphs and then displaying the results. As previous studies have demonstrated the efficiency and effectiveness of LLMs for graph learning \cite{liu2024can, chen2024exploring}, using ADA with GPT-4 is well-suited for understanding and implementing the rigorous mathematical computations of the proposed graph algorithms for constructing knowledge graphs. Combining these advanced graph algorithms with the language processing capabilities of GPT-4 ensures that the constructed knowledge graphs are both precise and contextually accurate, leading to improved performance in fact-checking. This hybrid approach leverages the strengths of both components, potentially achieving better results than relying on either the graph-based or language-based methods alone.

It is challenging for GPT-4 to just learn algorithms to construct accurate semantic health knowledge graphs from scratch. To address this, TrumorGPT employs few-shot learning to quickly adapt and generalize from a small number of examples. Suppose we are given a set of $N$ examples $\{(x_i, G_i)\}^N_{i=1}$, where $x_i$ denotes a query, either an article or a single sentence, and $G_i=\{E_i,R_i,F_i\}$ represents the corresponding optimized semantic health knowledge graph constructed from $x_i$ (see Section \ref{sec:knowledge_G}). Our goal is to develop task-agnostic learning strategies for GPT-4 that generalize well to unseen scenarios. Given a new input query $\hat{x}_i$, the objective is to accurately build the corresponding knowledge graph $\hat{G}_i$. As such, we first instruct GPT-4 on how to leverage the algorithms to optimize the process of creating accurate and contextually relevant knowledge graphs. Then, we provide a few examples with input queries and their designated knowledge graphs to demonstrate how this process can be effectively executed. Meanwhile, during the training stage, we can continually fine-tune the model to improve its performance by providing more precise and detailed instructions.

Through the few-shot learning approach with the topic-enhanced sentence centrality and topic-specific TextRank algorithms, GPT-4 enables TrumorGPT to identify and extract key phrases and sentences. These elements highlight the central ideas within the text, serving as the building blocks for the semantic health knowledge graph. By organizing these elements into a graph structure, TrumorGPT can represent the relationships and hierarchies among the various pieces of information. This approach is particularly necessary when dealing with extensive user inputs, such as full-length articles rather than single-sentence queries. It allows TrumorGPT to quickly understand and adapt to the task of knowledge graph construction by learning from a small, representative set of examples. These examples train the framework to identify the core ideas from articles, which is an important skill when the input is lengthy and complex. By optimizing the learning curve of the GPT-4, TrumorGPT can efficiently generalize the process of semantic health knowledge graph construction to new and diverse inputs, eliminating the need for exhaustive retraining on large datasets while ensuring the knowledge graph remains focused and relevant, even when derived from extensive sources of text.

\subsection{Graph-Based Retrieval-Augmented Generation}
Due to the non-deterministic behavior of GPT-4, it is difficult to predict what will be generated. Even with clear instructions, the generated result may not always meet the desired outcome. GraphRAG fundamentally enhances the fact-checking process in the TrumorGPT framework by augmenting the capabilities of GPT-4. The technical function of GraphRAG lies in its ability to dynamically access and incorporate external data, functioning like a database of semantic health knowledge graphs for fact-checking. GPT-4 serves as the foundation for conducting semantic similarity analysis in TrumorGPT, which compares user queries against a large corpus of text to identify relevant information in GraphRAG. This process is critical for ensuring that the semantic knowledge graph reflects the content of the user query accurately. Given a user input query $x$ and an external knowledge base consisting of $N$ knowledge graphs $\{G_i\}^N_{i=1}$, the goal of GraphRAG in TrumorGPT is to verify $x$ based on $\{G_i\}^N_{i=1}$. That is, if TrumorGPT is viewed as a function, then we have
\begin{equation*}
\text{TrumorGPT}(x,\{G_i\}^N_{i=1})=y,
\end{equation*}
where $y\in \{\text{True},\text{False}, \text{Undetermined}\}.$

To achieve the task, the user query $x$ is first transformed into a query knowledge graph $G_x = \{ E_x, R_x, F_x \}$ by identifying the relevant entities and relationships contained in the query. Next, a set of knowledge graphs $\{G_i\}^N_{i=1}$ is retrieved from the database based on their semantic similarity to $G_x$. This retrieval process relies on semantic similarity measures between the query graph $G_x$ and each knowledge graph $G_i$ in the database. The similarity score $S(G_x, G_i)$ can be computed using various graph similarity metrics, such as graph edit distance, subgraph isomorphism, or embedding-based approaches. Finally, TrumorGPT determines the answer based on the similarity scores. For instance, if there exists at least one knowledge graph $G_i$, $1\leq i\leq N$, such that $S(G_x, G_i)\leq \theta$, where $\theta$ is a threshold, then TrumorGPT outputs ``True''. This approach integrates the structured representation of semantic knowledge graphs with the generative capabilities of GPT-4, enhanced by GraphRAG, to deliver accurate and contextually grounded answers to health-related queries. In Fig. \ref{fig:graphrag}, we illustrate the workflow of GraphRAG in TrumorGPT.

\begin{figure}[tbp]
    \centering
    \includegraphics[width=\linewidth]{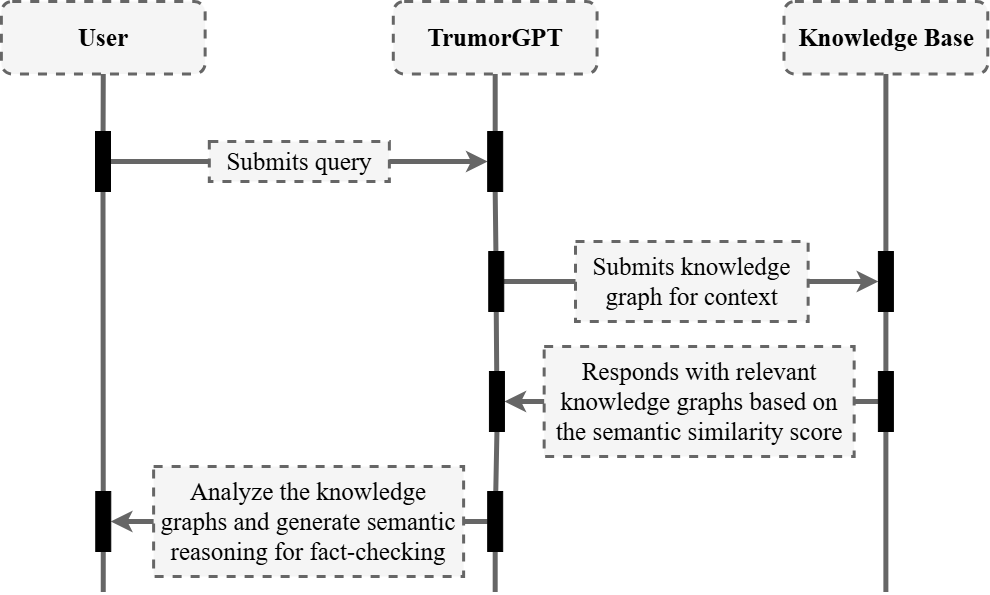}
    \caption{A sequence diagram illustrating the workflow of GraphRAG.}
    \label{fig:graphrag}
\end{figure}

In TrumorGPT, we integrate subgraph isomorphism and Jaccard similarity to define a scoring function $S(G_x, G_i)$ that quantifies the similarity between $G_x$ and $G_i$. We first define the set of consecutive triples extracted from a graph $G_i$ as:
$$
T_i = \{\, (e_1, r, e_2) \mid \text{there exists a path } e_1 \xrightarrow{r} e_2 \text{ in } G_i \,\}.
$$
To compare $G_x$ and $G_i$, we evaluate the overlap between their sets of consecutive triples. Let $f(t)$ denote a weighting function for a triple $t$ (e.g., derived from the centrality scores of the entities and relations in $t$). The similarity score $S(G_x, G_i)$ is then defined as:
\begin{align}
\label{eq:sim_score}
    S(G_x, G_i) = \frac{\sum_{t \in T_x \cap T_i} f(t)}{\sum_{t \in T_x \cup T_i} f(t)}.
\end{align}
In the simplest case, if we set $f(t)=1$ for all $t$ (i.e., treating every consecutive triple as equally important), (\ref{eq:sim_score}) reduces to the standard Jaccard similarity:
$$
S(G_x, G_i) = \frac{|T_x \cap T_i|}{|T_x \cup T_i|}.
$$

For illustration, consider a user query from which the following triples are extracted:
$$
T_x = \{\, (\text{Moderna},\, \text{prevents},\, \text{COVID}),\; (\text{Pfizer},\, \text{treats},\, \text{cancer}) \,\}, 
$$
and we have the following triples from a relevant semantic health knowledge graph $G_i$ in the knowledge base:
$$
T_i = \{\, (\text{Moderna},\, \text{prevents},\, \text{COVID}),\; (\text{A.Z.},\, \text{prevents},\, \text{flu}) \,\}.
$$
Then, we have
$$
T_x \cap T_i = \{\, (\text{Moderna},\, \text{prevents},\, \text{COVID}) \,\},
$$
and 
\begin{align*}
T_x \cup T_i = &\{\, (\text{Moderna},\, \text{prevents},\, \text{COVID}),\\ 
&(\text{Pfizer},\, \text{treats},\, \text{cancer}),\\ 
&(\text{A.Z.},\, \text{prevents},\, \text{flu}) \,\}.    
\end{align*}
Thus, the similarity score is:
$$
S(G_x, G_i) = \frac{1}{3} \approx 0.33.
$$

This scoring function quantifies the proportion of consecutive triples shared between the two graphs relative to the total number of unique triples. By incorporating a weighting function $f(t)$, the score can be refined to prioritize more important triples, such as those extracted from sentences with higher centrality or TST scores. Thus, it can improve the accuracy and reliability of the fact-checking process within our GraphRAG framework.

The hallucination issue in GPT-4 refers to instances where the model generates responses that are factually incorrect or not grounded in reality, often due to its dependence on static training data that may not be up-to-date or comprehensive. GraphRAG can be employed as a solution to mitigate this problem in TrumorGPT. GraphRAG allows GPT-4 to access an updated knowledge base that includes the latest news and information, all organized in the form of structured semantic knowledge graphs. In particular, TrumorGPT can be technically enhanced by GraphRAG through the following tasks:
\begin{itemize}
    \item Query processing and retrieval: Upon receiving a user query, TrumorGPT first processes it to understand the context and intent. With GraphRAG, the framework then searches the knowledge base, composed of semantic health knowledge graphs that encapsulate the most recent medical and health knowledge in a structured format, to find health information related to the query.
    \item Reference and verification: If the knowledge base contains graphs that directly correspond to the input query, TrumorGPT leverages this information to verify the truthfulness of the query. This process of direct verification depends on how relevant and recent the knowledge graphs are in relation to the query.
    \item Knowledge graph expansion: In situations where the knowledge base lacks a specific knowledge graph for direct verification, TrumorGPT enhances the existing graph related to the input query by adding new vertices and edges, integrating additional information and context. This expansion broadens the scope of the knowledge graph for a more comprehensive fact-checking process.
    \item Guidance for decision making: If there is not enough information for exact fact-checking, TrumorGPT provides additional, relevant knowledge from the knowledge base. This information acts as a guide for users, helping them make decisions regarding the veracity of their query.
\end{itemize}

GraphRAG effectively equips GPT-4 with the ability to reference the most current and relevant data, going beyond its pretrained static dataset. This dynamic interaction with an evolving knowledge base significantly reduces instances of factual inaccuracies in the responses generated by GPT-4, making TrumorGPT a more reliable and up-to-date fact-checking tool. If the input query can be fact-checked using the semantic health knowledge graphs in GraphRAG as false or true, then TrumorGPT provides evidence based on these knowledge graphs to support the claim (true or false) with semantic reasoning. If the query is undetermined, TrumorGPT offers relevant information from its knowledge base, which is also constructed from semantic health knowledge graphs, along with semantic reasoning, enabling the user to decide the trustworthiness of the query.

\section{Performance Evaluation}
\label{sec:exp}
In this section, we examine the effectiveness of our TrumorGPT framework in validating the truthfulness of health news content as an application of fact-checking.

\subsection{Experimental Setup}
We enhance TrumorGPT with up-to-date knowledge by incorporating Resource Description Framework (RDF) triples from the latest-core collection of DBpedia \cite{auer2007dbpedia}, which automatically extracts structured data from Wikipedia reflecting the most recent updates. Focusing on the DBpedia Ontology dataset of triples, we construct a knowledge base with information exclusively in English. Given that the existing pretraining of GPT-4 includes knowledge up to December 2023, we concentrate on training it with data postdating this period to ensure knowledge of TrumorGPT remains updated. To address the limitations of GPT-4 in processing extensive external data, we selectively train it on triples related to health information, ensuring relevance and efficiency of the framework in a domain where accuracy is specifically consequential due to its impact on public health and safety. In particular, we filter the dataset using labels to include only resources with health-related keywords such as ``health'', ``medical'', ``medicine'', ``disease'', ``pandemic'', and ``epidemic'' (string matching is case-insensitive).

To evaluate the fact-checking performance of our proposed TrumorGPT on public health claims, we compare it with four advanced language models: GPT-3.5 \cite{brown2020language, ouyang2022training}, GPT-4, LLaMA 3.2 \cite{touvron2023llama}, PaLM 2 \cite{anil2023palm}, Claude 3.5 Sonnet, and Gemini 1.5 \cite{reid2024gemini}. We use widely adopted evaluation metrics, including accuracy, precision, recall, and F1-score, to assess fact-checking performance. In our evaluations, true positives and true negatives denote correctly identified claims, while false positives and false negatives represent misclassifications. Accuracy reflects the overall proportion of correct predictions, precision measures the correctness of positive predictions, recall indicates the ability of the model to identify true claims, and F1-score balances precision and recall.

\subsection{Evaluation Approach}
As the 2024 United States presidential election approaches, our focus in this study is on health policy and political news, recognizing the critical role of fact-checking in ensuring informed voter decisions and maintaining the integrity of the electoral process. Fact-checking is important before an election as it helps to clarify the positions and claims of candidates, particularly regarding health policies, contributing to a more transparent and fair political discourse. Health policy is a crucial topic because it directly impacts public well-being, healthcare access, and the management of health crises, such as the COVID-19 pandemic. Misinformation or unclear statements about health policies can lead to significant public confusion and potentially harmful decisions. By focusing on health policy news, we aim to ensure voters are accurately informed about the stances of candidates on critical health issues, highlighting the importance of transparent, evidence based discussions in the political arena, particularly on matters that have far-reaching consequences for public health.

In this regard, we focus on PolitiFact, a well-known fact-checking organization that evaluates the accuracy of claims made by politicians, candidates, and others involved in United States politics. While PolitiFact covers a wide range of political statements, we particularly concentrate on issues related to health care, evaluating the politics and policies affecting public health in the United States. PolitiFact conducts fact-checking by researching statements and rating their accuracy, from policy debates to campaign assertions. Due to its comprehensive and methodical approach to evaluating political statements, we consider the fact-checking outcomes by PolitiFact as our gold standard for assessing the performance of TrumorGPT. The well-established framework of PolitiFact provides a robust benchmark for comparing and validating the fact-checking capabilities of TrumorGPT in the area of political news on health issues, especially in the context of the upcoming United States presidential election.

PolitiFact employs a six-category rating system to evaluate the accuracy of statements: ``True'' for statements that are completely accurate with no significant omissions, ``Mostly True'' for statements that are accurate but require additional information or clarification, ``Half True'' for partially accurate statements that omit important details or take things out of context, ``Mostly False'' for statements containing some truth but lacking critical facts, ``False'' for completely inaccurate statements, and ``Pants on Fire'' for not only false but also ridiculous claims. LLMs may struggle with the fine distinctions between these categories. Therefore, we simplify these into a binary system of ``True'' or ``False'' for practicality. We categorize ``True'', ``Mostly True'', and ``Half True'' as ``True'', indicating a tendency towards accuracy, and ``Mostly False'', ``False'', and ``Pants on Fire'' as ``False'', denoting significant inaccuracy or outright falsehood. This binary classification method simplifies fact-checking into just true or false categories, which is important for ensuring both accuracy and speed in the verification of information.

\subsection{Results}

We compare the performance of TrumorGPT with five other advanced language models using a set of 600 statements from the ``Health Care'' and ``Coronavirus'' categories on PolitiFact, evenly split with 300 true and 300 false claims. As shown in Table \ref{tab:acc_per}, all models achieve stable accuracy above 70\% in fact-checking health-related claims, with TrumorGPT standing out as the most accurate at 88.5\%. Meanwhile, the precision, recall, and F1-score for all LLMs indicate a generally robust performance in fact-checking health-related claims, with TrumorGPT leading across all metrics. An analysis of these evaluation metrics further reveals that for all LLMs, precision exceeds accuracy, which in turn is higher than recall, and the F1-score is the lowest. This ordering suggests that while the LLMs are highly effective at ensuring that positive predictions are correct, they tend to miss some true positive cases, resulting in an F1-score that reflects this trade-off. 

Although GPT-4 demonstrates competitive performance, TrumorGPT outperforms it by building on its foundation with specialized health-specific enhancements. In particular, training data of GPT-4 only extends to December 2023, limiting its awareness of recent developments, while TrumorGPT augments this with current knowledge through semantic health knowledge graphs in GraphRAG, providing access to the latest health information and policies. Additionally, other general-purpose LLMs in this evaluation lack the domain-specific knowledge that TrumorGPT possesses in health and medical misinformation. This specialization enables TrumorGPT to better identify nuances in health-related claims, improving its accuracy and reliability in fact-checking health information.

\begin{table}[tb]
\centering
\renewcommand{\arraystretch}{1.5}
\caption{Performance Comparison of TrumorGPT and Other LLMs for Fact-Checking Health-Related Claims}
\begin{tabular}{lcccc}
\hline
LLMs & Accuracy & Precision & Recall & F1-score \\
\hline
GPT-3.5 & 72.7\% & 75.8\% & 66.7\% & 70.9\% \\
GPT-4 & 83.3\% & 85.7\% & 80.0\% & 82.8\% \\
LLaMA 3.2 & 81.8\% & 83.0\% & 80.0\% & 81.5\% \\
PaLM 2 & 76.8\% & 77.9\% & 75.0\% & 76.4\% \\
Claude 3.5 Sonnet & 77.2\% & 78.4\% & 75.0\% & 76.7\% \\
Gemini 1.5 & 81.7\% & 83.9\% & 78.3\% & 81.0\% \\
TrumorGPT & \textbf{88.5\%} & \textbf{91.4\%} & \textbf{85.0\%} & \textbf{88.1\%} \\
\hline
\end{tabular}
\label{tab:acc_per}
\end{table}

Measuring the average number of sentences generated in fact-checking responses provides insights into the conciseness and clarity of each model. A more concise response with fewer sentences can indicate efficient information retrieval, making it easier for users to quickly understand the fact-checking result without sifting through excessive details. As shown in Fig. \ref{fig:avg_sen}, TrumorGPT has the lowest average sentence count at 2.8 sentences, suggesting that it is highly efficient in delivering clear, fact-based responses. We further examine the relationship between prediction accuracy and explanation length for fact-checking responses across different LLMs. Fig. \ref{fig:acc_sen} shows that the number of sentences per response for all LLMs falls within a range of 2 to 6 sentences. In Fig. \ref{fig:acc_words}, we explore the average word count per response. Notably, TrumorGPT achieves the highest accuracy while delivering the most concise explanations among the models. This advantage is due to the integration of semantic health knowledge graphs in TrumorGPT, enabling access to precise, current information. By focusing only on relevant details, TrumorGPT enhances accuracy without increasing response length, an especially valuable trait for fact-checking in public health domains where users seek quick and reliable information.

\begin{figure}[tb]
    \centering
    \includegraphics[width=\linewidth]{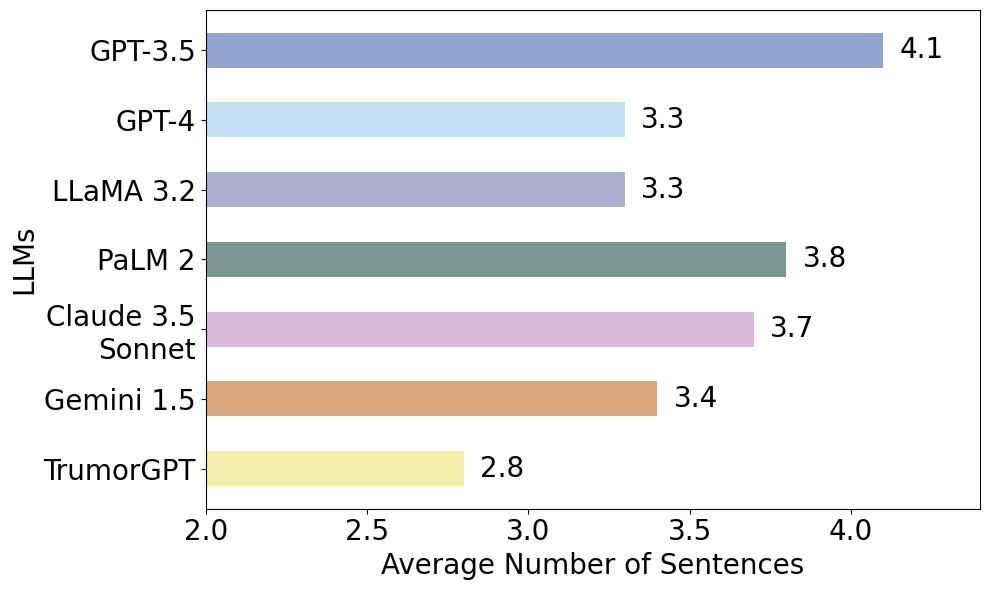}
    \caption{Comparison of average number of sentences generated by LLMs for fact-checking responses.}
    \label{fig:avg_sen}
\end{figure}

\begin{figure}[tp]
    \centering
    \subfigure[Comparison of fact-checking accuracy and number of sentences per response generated by various LLMs.]
    {
        \includegraphics[width=\linewidth]{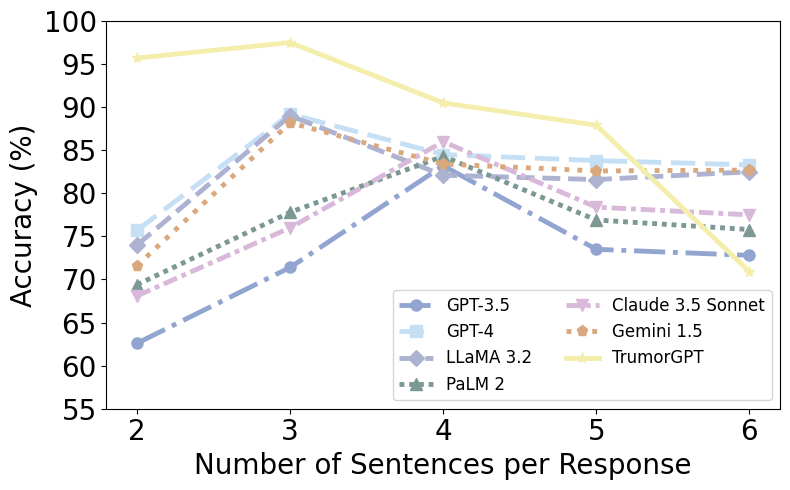}
        \label{fig:acc_sen}
    }
    \subfigure[Comparison of fact-checking accuracy and average word count per response generated by various LLMs.]
    {
        \includegraphics[width=\linewidth]{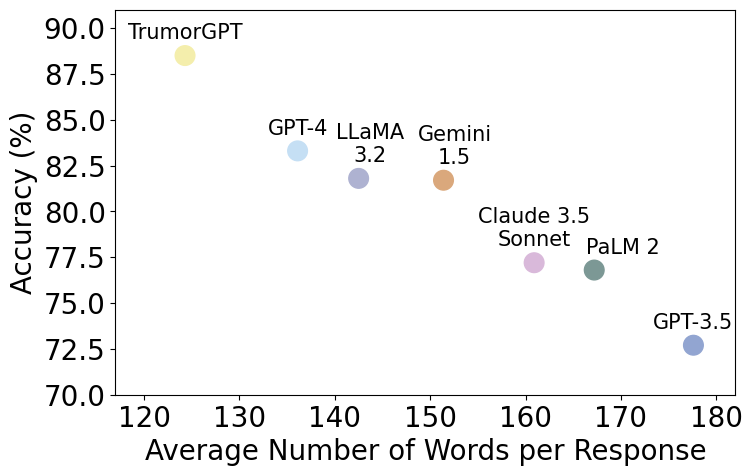}
        \label{fig:acc_words}
    }
    \caption{Correlation between prediction accuracy and explanation length in fact-checking responses across various LLMs.}
    \label{fig:acc_len}
\end{figure}

\subsection{TrumorGPT for Fact-Checking of Public Health Claims}
We specifically examine the performance and behavior of TrumorGPT in fact-checking public health claims. Fig. \ref{fig:confusion_mat} displays the normalized confusion matrix, illustrating the fact-checking performance of TrumorGPT. We observe that the model correctly identifies 85\% of true statements (true positives) and 92\% of false statements (true negatives), demonstrating its proficiency in fact-checking. Meanwhile, TrumorGPT achieves a slightly better performance in identifying false statements over true ones. One possible explanation for this could be the nature of the data or the way misinformation is structured. False statements often present specific, identifiable inaccuracies or exaggerations, which the model can detect using the semantic health knowledge graph. In contrast, verifying the truthfulness of a statement often requires a comprehensive understanding and cross-referencing of facts, which can be more challenging and may lead to a marginally lower true positive rate. This phenomenon reflects a common pattern in fact-checking, where disproving a statement is often more feasible than definitively proving its accuracy, given the more straightforward nature of evidence required to refute a falsehood as opposed to the extensive and sometimes ambiguous evidence needed to confirm the truth. 

\begin{figure}[tb]
    \centering
    \includegraphics[width=.9\linewidth]{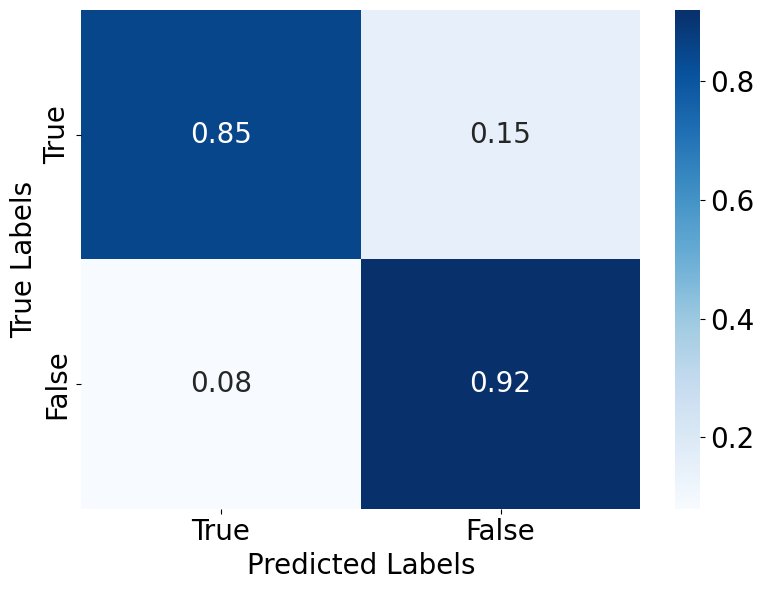}
    \caption{A normalized confusion matrix for visualizing the performance of TrumorGPT in binary fact-checking.}
    \label{fig:confusion_mat}
\end{figure}

We further evaluate TrumorGPT on fact-checking the 600 statements using the original six-category rating system employed by PolitiFact, with 100 claims per category. Fig. \ref{fig:confusion_mat2} shows the overall accuracy of TrumorGPT at 49.3\% across the six-category system, a notable reduction compared to its binary classification accuracy. TrumorGPT differentiates among the six categories by analyzing the completeness and quality of factual support in its semantic health knowledge graphs. When a claim is fully validated, with comprehensive connections among all relevant entities, it is classified as ``True''. If the graph supports the claim with nearly all critical details but lacks some clarifications, the claim is deemed ``Mostly True''. In contrast, if the graph shows that the claim is only partially supported, omitting several important details, it is labeled ``Half True''. When some elements of truth are present yet key facts are missing, the claim is categorized as ``Mostly False'', and if the graph entirely fails to support the claim, it is classified as ``False''. When the claim is not only unsupported but also contains absurd or exaggerated elements, it is rated ``Pants on Fire''.

Although the multi-class accuracy is significantly lower with this approach, we observe that TrumorGPT consistently classifies true statements within the ``True'', ``Mostly True'', or ``Half True'' categories, and false statements within the ``Mostly False'', ``False'', or ``Pants on Fire'' categories. It indicates that classifications of TrumorGPT are generally on the correct path, even if they are not perfectly aligned with the exact label. This is the reason why its overall accuracy for binary classification of true or false remains high at 88.5\%. A further trend reveals that when TrumorGPT encounters uncertainty, it often selects middle-ground categories like ``Half True'' or ``Mostly False''. This suggests that TrumorGPT takes a cautious approach, often leaning towards a generally correct direction rather than providing an exact answer. Therefore, TrumorGPT is optimized for binary classification, providing clear-cut true or false decisions, which is particularly valuable in today’s infodemic age where users benefit from straightforward and reliable fact-checking.

\begin{figure}[tb]
    \centering
    \includegraphics[width=\linewidth]{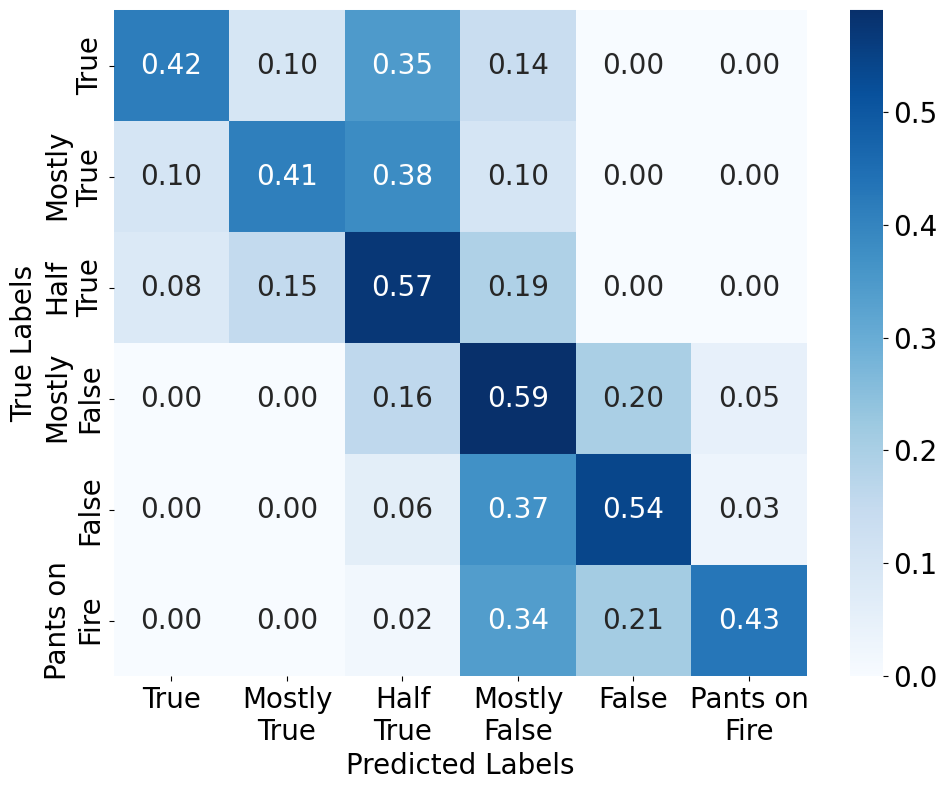}
    \caption{A normalized confusion matrix for visualizing the performance of TrumorGPT in multi-class fact-checking.}
    \label{fig:confusion_mat2}
\end{figure}

Another important aspect in our analysis is to evaluate how TrumorGPT leverages semantic health knowledge graphs in fact-checking public health claims. In previous simulations, it was observed that TrumorGPT achieves higher accuracy with shorter explanations, indicating that the use of semantic health knowledge graphs can effectively distill essential information for concise fact-checking responses. To further assess the performance of TrumorGPT in constructing these knowledge graphs, we select 50 claims from the 600 statements, evenly split between true and false claims. For each claim, PolitiFact provides an article of varying length to support or refute the statement. We modify these articles by either removing explicit fact-indicating information or altering keywords to introduce potential misdirection. Then, for each altered article, we instruct TrumorGPT to construct a semantic health knowledge graph and perform fact-checking. 

As shown in Fig. \ref{fig:kg_words}, the size of the knowledge graphs ranges from 6 to 18 vertices, with the number of vertices generally increasing linearly with the length of the article. This trend suggests that TrumorGPT constructs more complex graphs for longer articles to capture the additional details. However, as the article length increases, the error rate also rises. This is primarily due to the fact that with more content, TrumorGPT may struggle to identify key points, leading to confusion and potential hallucination in its responses. Therefore, while TrumorGPT effectively uses knowledge graphs for accurate fact-checking in shorter articles, longer articles may still require further refinement in the model to maintain accuracy. Table \ref{tab:example} presents two illustrative examples of TrumorGPT applied to fact-checking, detailing the complete outcomes, including the final semantic health knowledge graphs and the corresponding semantic reasoning generated by TrumorGPT.

\begin{figure}[tb]
    \centering
    \includegraphics[width=\linewidth]{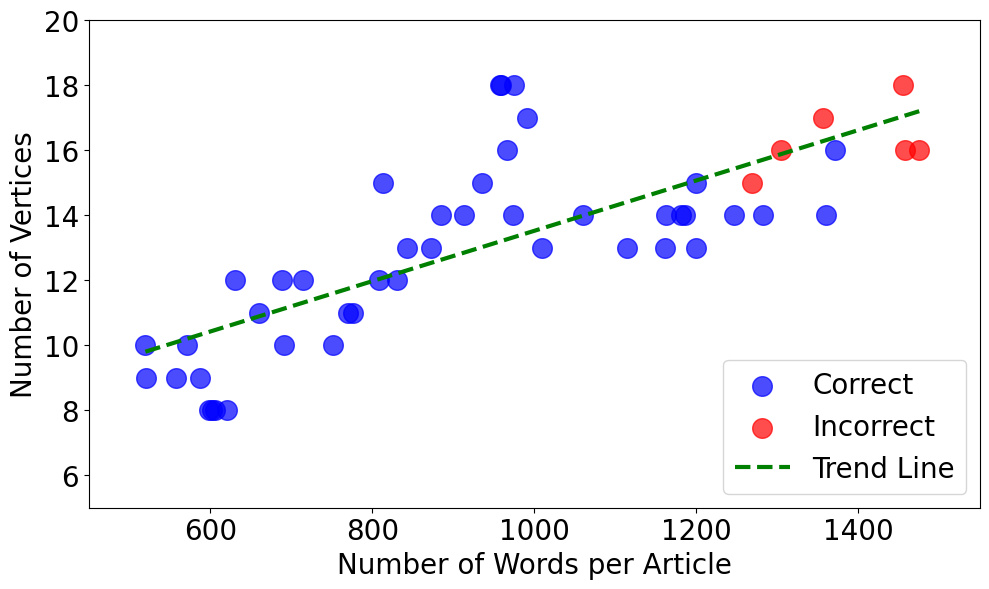}
    \caption{Performance of TrumorGPT in constructing semantic health knowledge graphs for fact-checking public health claims. Blue points represent correctly fact-checked outcomes, while red points indicate incorrect outcomes. The green dotted line shows the trend of increasing knowledge graph size (number of vertices) as the number of words per article grows.}
    \label{fig:kg_words}
\end{figure}

\begin{table*}[tb]
    \centering
    \renewcommand{\arraystretch}{1.5}
        \centering
        \newcolumntype{P}[1]{>{\arraybackslash\centering}p{#1}}
        \caption{Examples of fact-checking accuracy of TrumorGPT, illustrating two distinct cases where TrumorGPT successfully verifies health-related queries in the context of United States politics. \label{tab:example}}
        \begin{tabular}{P{.15\textwidth}P{.35\textwidth}P{.30\textwidth}P{.08\textwidth}}
            \toprule
            \textbf{Query} & \textbf{Semantic Health Knowledge Graph} & \textbf{Semantic Reasoning} & \textbf{Ground Truth} \tabularnewline  \midrule
             The number of COVID-19 deaths recorded so far in 2021 has surpassed the total for 2020. & 
             \adjustbox{valign=t}{\begin{tikzpicture}[auto, node distance=1cm,>=latex',thick, every text node part/.style={align=center}]
                \node [block, text width = 1cm](a){CDC};
                \node [below = 1cm of a](a1){};
                \node [block, text width = 1.6cm, left = of a1](b){U.S. Deaths in 2020};
                \node [block, text width = 1.6cm, right = of a1](c){U.S. Deaths in 2021};
                \node [block, text width = 1.8cm, below = 1cm of a1](d){More Deaths in 2021 than 2020};
                \node [below = 1cm of d](a2){};
                \node [block, text width = 1.6cm, left = of a2](e){National Outlets};
                \node [block, text width = 1.8cm, right = of a2](f){Johns Hopkins Dashboard};
                \draw [->,thick] (a) -- (b) node[left = 0.3cm, midway]{reported 385,443}; 
                \draw [->,thick] (a) -- (c) node[right = 0.3cm, midway]{reported 446,197}; 
                \draw [->,thick] (b) -- (d) node[left = 0.3cm, midway]{implies};
                \draw [->,thick] (c) -- (d) node[right = 0.3cm, midway]{implies};
                \draw [->,thick] (a) -- (d) node[midway, midway]{agrees\\with};
                \draw [->,thick] (e) -- (d) node[left = 0.3cm, midway]{reported};
                \draw [->,thick] (f) -- (d) node[right = 0.3cm, midway]{shows};
                \draw [->,thick] (e) -- (f) node[midway, midway]{used data from};
            \end{tikzpicture}} & \texttt{The statement is true. CDC data reports 385,443 COVID-19 deaths in 2020 and 446,197 deaths in 2021, confirming that more deaths occurred in 2021. This conclusion is supported by national outlets like The New York Times, which used data from the Johns Hopkins University.} & True \tabularnewline \tabularnewline
             Florida sort of had a mandate because they were giving the vaccine, they were demanding everybody take the vaccine. & \adjustbox{valign=t}{\begin{tikzpicture}[auto, node distance=1cm,>=latex',thick, every text node part/.style={align=center}]
                \node [block, text width = 1.6cm](a){Governor Ron DeSantis};
                \node [block, text width = 1.6cm, below = 1.7cm of a](b){Florida};
                \node [block, text width = 1.6cm, right = 1.5cm of b](c){Legislation and Bills};
                \node [block, text width = 1.6cm, right = 1.5cm of a](d){Vaccine Mandates};
                \draw [->,thick] (b) -- (a) node[left, midway]{governed by}; 
                \draw [->,thick] (c) -- (b) node[midway, midway]{enforced in}; 
                \draw [->,thick] (c) -- (d) node[right, midway]{blocked}; 
                \draw [->,thick] (a) -- (d) node[midway, midway]{public statements\\opposing}; 
                \draw [->,thick] (a) -- (c) node[midway, midway]{signed}; 
            \end{tikzpicture}} & \texttt{The statement is false. Governor Ron DeSantis publicly opposed vaccine mandates and signed legislation to block such mandates in Florida. These laws were enforced, ensuring that vaccines remained voluntary, not mandated by the state.} & False \tabularnewline \bottomrule
        \end{tabular}
\end{table*}

In addition to its impressive performance in the verification of news truthfulness, TrumorGPT also excels in validating factual statements, a capability that is greatly enhanced by the integration of the GPT-4 model. For example, TrumorGPT can accurately verify the statement, ``A balanced diet can help manage diabetes'', by referencing current medical research and guidelines from health organizations that highlight how maintaining a balanced diet helps in controlling blood sugar levels, reducing the risk of complications associated with diabetes. This demonstrates the ability of TrumorGPT to apply common sense reasoning to health and medical knowledge effectively. The advanced language comprehension and analysis capabilities of GPT-4 boost the accuracy and reliability of TrumorGPT in fact-checking. This enables TrumorGPT to effectively and accurately recognize the truth in various factual statements, making it a valuable tool in fact verification. When asked about technical health knowledge, such as diagnosing specific conditions, TrumorGPT also provides a disclaimer and suggests users to seek professional advice from a medical doctor to ensure accurate diagnosis and appropriate treatment.

\section{Discussion}
\label{sec:dis}
In this section, we provide a technical analysis of contemporary RAG techniques by comparing HybridRAG \cite{sarmah2024hybridrag} and LightRAG \cite{guo2024lightrag} with TrumorGPT, focusing on their methodologies, reasoning capabilities, and efficiency in health-related fact-checking. We also outline potential enhancements for TrumorGPT inspired by these state-of-the-art approaches.

HybridRAG \cite{sarmah2024hybridrag} employs a dual-retrieval mechanism by integrating vector-based text retrieval with graph-based knowledge extraction. Technically, it first performs a conventional embedding search to identify relevant textual passages, then simultaneously accesses a structured knowledge graph to extract entity relationships and contextual links. This approach offers greater coverage and flexibility compared to the GraphRAG employed by TrumorGPT. For fact-checking scenarios, HybridRAG can retrieve textual evidence even when a claim does not neatly match an existing knowledge graph entry. For example, if a health claim is phrased unconventionally or involves implicit context, the vector search component can still identify pertinent data that might not be captured by a direct graph lookup. The inclusion of a knowledge graph within HybridRAG further ensures that precise fact relationships are maintained, which keeps the answers accurate. 

In contrast, the GraphRAG employed by TrumorGPT is finely tuned for precision within its curated semantic health knowledge graph. When the query aligns well with the graph, meaning that the claim clearly maps to existing entities and relationships, the system consistently delivers highly faithful and contextually precise fact-checking outcomes. Although its targeted approach may sometimes result in an ``undetermined'' response in cases where the specific fact is not explicitly represented, this selectivity also serves to avoid the introduction of extraneous details that could otherwise complicate the fact verification process. This disciplined retrieval strategy ensures that TrumorGPT remains competitive by providing reliable outputs with minimal noise, which is critical in domains where accuracy is paramount.

LightRAG \cite{guo2024lightrag} adopts an efficient retrieval framework by organizing information into a dual-level knowledge graph, where low-level retrieval targets specific facts and high-level retrieval captures broader thematic context. The design of LightRAG offers enhanced efficiency and broader retrieval capabilities by employing dual-level retrieval and an incremental graph update mechanism. These features are especially useful in a dynamic field like health, where new studies and news frequently appear. For example, LightRAG can rapidly integrate fresh medical data and capture both direct relationships and broader thematic connections, enabling the retrieval of comprehensive evidence sets, such as an effect chain of a drug from administration to outcome, which may enrich fact-checking explanations. 

In contrast, GraphRAG of TrumorGPT is optimized for precision and operates on a curated set of semantic health knowledge graphs, consistently delivering accurate and focused true or false decisions when the claim maps clearly onto its graph. While LightRAG provides a broader perspective in cases that require complex, multi-step reasoning, TrumorGPT remains highly effective in its domain, ensuring accuracy while maintaining efficient processing and reducing unnecessary information. Adopting elements from LightRAG, such as efficient incremental updates and dual-level retrieval, could further improve TrumorGPT without compromising its strong fact-based verification process.

Drawing from the strengths of both HybridRAG and LightRAG, TrumorGPT can be further enhanced by integrating a hybrid retrieval mechanism that combines the precision of graph-based retrieval with the expansive reach of vector-based contextual search. This dual approach would allow TrumorGPT to not only access highly accurate, structured data from its semantic health knowledge graphs but also capture nuanced, implicit contextual cues present in unstructured text. For instance, when a health claim is expressed in unconventional language or omits explicit entity mentions, the vector search component can retrieve relevant background information and subtle interconnections that a purely graph-based system might overlook. By effectively merging these two retrieval paradigms, the system would be capable of assembling a more comprehensive evidence base, thereby reducing the likelihood of missing pertinent details. Additionally, incorporating incremental graph update techniques would enable TrumorGPT to seamlessly integrate new medical studies and real-time health news without the need for extensive re-indexing or system downtime. This continuous update process is particularly crucial in the dynamic field of health, where the rapid emergence of new data can significantly impact the accuracy of fact-checking outcomes. 

In addition to efficient updating, implementing a dual-level retrieval strategy can further enhance contextual reasoning. Under this strategy, an initial coarse-grained retrieval would rapidly identify broad thematic and relational contexts, while a subsequent fine-grained search would extract specific factual details. This layered retrieval approach not only ensures that both high-level and granular information is captured but also supports more robust multi-hop reasoning by effectively bridging the gap between abstract concepts and concrete facts. Thus, these enhancements would bolster the fact-checking framework of TrumorGPT, striking an optimal balance between precision, efficiency, and comprehensive reasoning.

\section{Conclusion}
\label{sec:conclusion}
Our proposed framework, TrumorGPT, leverages a large language model for accurate verification of health-related statements in fact-checking. Effectively combining semantic health knowledge graphs with retrieval-augmented generation, TrumorGPT exhibits robust performance in fact-checking, especially relevant to the health policy and politics of the United States in preparation for the 2024 presidential election. In particular, we demonstrate how topic-enhanced sentence centrality and topic-specific TextRank, enhanced with few-shot learning, optimizes the formation of semantic health knowledge graphs while retrieval-augmented generation refines the GPT-4 model with the most recent and updated health information, addressing the hallucination issue in large language models. For future work, an interesting direction involves expanding the capabilities of TrumorGPT to tackle additional fact-checking tasks like identifying the origins of rumors or detecting spam content.

\section*{Acknowledgment}
The authors are grateful for helpful discussions on the subject with D.-M. Chiu.

\bibliographystyle{IEEEtran}
\bibliography{IEEEabrv,references}

\begin{thebibliography}{10}
\providecommand{\url}[1]{#1}
\csname url@samestyle\endcsname
\providecommand{\newblock}{\relax}
\providecommand{\bibinfo}[2]{#2}
\providecommand{\BIBentrySTDinterwordspacing}{\spaceskip=0pt\relax}
\providecommand{\BIBentryALTinterwordstretchfactor}{4}
\providecommand{\BIBentryALTinterwordspacing}{\spaceskip=\fontdimen2\font plus
\BIBentryALTinterwordstretchfactor\fontdimen3\font minus \fontdimen4\font\relax}
\providecommand{\BIBforeignlanguage}[2]{{%
\expandafter\ifx\csname l@#1\endcsname\relax
\typeout{** WARNING: IEEEtran.bst: No hyphenation pattern has been}%
\typeout{** loaded for the language `#1'. Using the pattern for}%
\typeout{** the default language instead.}%
\else
\language=\csname l@#1\endcsname
\fi
#2}}
\providecommand{\BIBdecl}{\relax}
\BIBdecl

\bibitem{acemoglu2023model}
D.~Acemoglu, A.~Ozdaglar, and J.~Siderius, ``A model of online misinformation,'' \emph{Review of Economic Studies}, p. rdad111, 2023.

\bibitem{hang2023mega}
C.~N. Hang, P.-D. Yu, S.~Chen, C.~W. Tan, and G.~Chen, ``{MEGA}: Machine learning-enhanced graph analytics for infodemic risk management,'' \emph{IEEE Journal of Biomedical and Health Informatics}, vol.~27, no.~12, pp. 6100--6111, 2023.

\bibitem{siderius2021misinformation}
J.~Siderius, ``Misinformation: Strategic sharing, homophily, and endogenous echo chambers,'' Technical Report, National Bureau of Economic Research, Tech. Rep., 2021.

\bibitem{acemoglu2016network}
D.~Acemoglu, A.~Malekian, and A.~Ozdaglar, ``Network security and contagion,'' \emph{Journal of Economic Theory}, vol. 166, pp. 536--585, 2016.

\bibitem{acemoglu2011opinion}
D.~Acemoglu and A.~Ozdaglar, ``Opinion dynamics and learning in social networks,'' \emph{Dynamic Games and Applications}, vol.~1, pp. 3--49, 2011.

\bibitem{del2016spreading}
M.~Del~Vicario, A.~Bessi, F.~Zollo, F.~Petroni, A.~Scala, G.~Caldarelli, H.~E. Stanley, and W.~Quattrociocchi, ``The spreading of misinformation online,'' \emph{Proceedings of the National Academy of Sciences}, vol. 113, no.~3, pp. 554--559, 2016.

\bibitem{vosoughi2018spread}
S.~Vosoughi, D.~Roy, and S.~Aral, ``The spread of true and false news online,'' \emph{Science}, vol. 359, no. 6380, pp. 1146--1151, 2018.

\bibitem{tan2023contagion}
C.~W. Tan and P.-D. Yu, ``Contagion source detection in epidemic and infodemic outbreaks: Mathematical analysis and network algorithms,'' \emph{Found. Trends{\textregistered} Netw.}, vol.~13, no. 2-3, pp. 107--251, 2023.

\bibitem{luo2021spread}
H.~Luo, M.~Cai, and Y.~Cui, ``Spread of misinformation in social networks: Analysis based on {Weibo} tweets,'' \emph{Security and Communication Networks}, vol. 2021, no.~1, p. 7999760, 2021.

\bibitem{acemoglu2024simple}
D.~Acemoglu, ``The simple macroeconomics of {AI},'' National Bureau of Economic Research, Tech. Rep., 2024.

\bibitem{walter2020fact}
N.~Walter, J.~Cohen, R.~L. Holbert, and Y.~Morag, ``Fact-checking: A meta-analysis of what works and for whom,'' \emph{Political Communication}, vol.~37, no.~3, pp. 350--375, 2020.

\bibitem{alhindi2018your}
T.~Alhindi, S.~Petridis, and S.~Muresan, ``Where is your evidence: Improving fact-checking by justification modeling,'' in \emph{Proceedings of the First Workshop on Fact Extraction and Verification (FEVER)}, 2018, pp. 85--90.

\bibitem{vaughan2018making}
J.~W. Vaughan, ``Making better use of the crowd: How crowdsourcing can advance machine learning research,'' \emph{Journal of Machine Learning Research}, vol.~18, no. 193, pp. 1--46, 2018.

\bibitem{dafoe2021cooperative}
A.~Dafoe, Y.~Bachrach, G.~Hadfield, E.~Horvitz, K.~Larson, and T.~Graepel, ``Cooperative {AI}: Machines must learn to find common ground,'' 2021.

\bibitem{touvron2023llama}
H.~Touvron, T.~Lavril, G.~Izacard, X.~Martinet, M.-A. Lachaux, T.~Lacroix, B.~Rozi{\`e}re, N.~Goyal, E.~Hambro, F.~Azhar, A.~Rodriguez, A.~Joulin, E.~Grave, and G.~Lample, ``{LLaMA}: Open and efficient foundation language models,'' \emph{arXiv:2302.13971}, 2023.

\bibitem{ouyang2022training}
L.~Ouyang, J.~Wu, X.~Jiang, D.~Almeida, C.~Wainwright, P.~Mishkin, C.~Zhang, S.~Agarwal, K.~Slama, A.~Ray \emph{et~al.}, ``Training language models to follow instructions with human feedback,'' in \emph{Advances in Neural Information Processing Systems}, vol.~35, 2022, pp. 27\,730--27\,744.

\bibitem{brown2020language}
T.~Brown, B.~Mann, N.~Ryder, M.~Subbiah, J.~D. Kaplan, P.~Dhariwal, A.~Neelakantan, P.~Shyam, G.~Sastry, A.~Askell \emph{et~al.}, ``Language models are few-shot learners,'' in \emph{Advances in Neural Information Processing Systems}, vol.~33, 2020, pp. 1877--1901.

\bibitem{anil2023palm}
R.~Anil, A.~M. Dai, O.~Firat, M.~Johnson, D.~Lepikhin, A.~Passos, S.~Shakeri, E.~Taropa, P.~Bailey, Z.~Chen \emph{et~al.}, ``{PaLM} 2 technical report,'' \emph{arXiv:2305.10403}, 2023.

\bibitem{openai2023gpt4}
OpenAI, ``{GPT-4} technical report,'' \emph{arXiv:2303.08774}, 2023.

\bibitem{zheng2019sentence}
H.~Zheng and M.~Lapata, ``Sentence centrality revisited for unsupervised summarization,'' \emph{arXiv:1906.03508}, 2019.

\bibitem{mihalcea2004textrank}
R.~Mihalcea and P.~Tarau, ``{TextRank}: Bringing order into text,'' in \emph{Proceedings of the 2004 Conference on Empirical Methods in Natural Language Processing}, 2004, pp. 404--411.

\bibitem{ceri2013introduction}
S.~Ceri, A.~Bozzon, M.~Brambilla, E.~Della~Valle, P.~Fraternali, S.~Quarteroni, S.~Ceri, A.~Bozzon, M.~Brambilla, E.~Della~Valle \emph{et~al.}, ``An introduction to information retrieval,'' \emph{Web Information Retrieval}, pp. 3--11, 2013.

\bibitem{kazemi2020biased}
A.~Kazemi, V.~P{\'e}rez-Rosas, and R.~Mihalcea, ``Biased {TextRank}: Unsupervised graph-based content extraction,'' \emph{arXiv:2011.01026}, 2020.

\bibitem{florescu2017positionrank}
C.~Florescu and C.~Caragea, ``{PositionRank}: An unsupervised approach to keyphrase extraction from scholarly documents,'' in \emph{Proceedings of the 55th Annual Meeting of the Association for Computational Linguistics (volume 1: long papers)}, 2017, pp. 1105--1115.

\bibitem{lewis2020retrieval}
P.~Lewis, E.~Perez, A.~Piktus, F.~Petroni, V.~Karpukhin, N.~Goyal, H.~K{\"u}ttler, M.~Lewis, W.-t. Yih, T.~Rockt{\"a}schel \emph{et~al.}, ``Retrieval-augmented generation for knowledge-intensive {NLP} tasks,'' in \emph{Advances in Neural Information Processing Systems}, vol.~33, 2020, pp. 9459--9474.

\bibitem{edge2024local}
D.~Edge, H.~Trinh, N.~Cheng, J.~Bradley, A.~Chao, A.~Mody, S.~Truitt, and J.~Larson, ``From local to global: A graph {RAG} approach to query-focused summarization,'' \emph{arXiv:2404.16130}, 2024.

\bibitem{rasmy2021med}
L.~Rasmy, Y.~Xiang, Z.~Xie, C.~Tao, and D.~Zhi, ``{Med-BERT}: Pretrained contextualized embeddings on large-scale structured electronic health records for disease prediction,'' \emph{NPJ Digital Medicine}, vol.~4, no.~1, p.~86, 2021.

\bibitem{de2014medical}
L.~De~Vine, G.~Zuccon, B.~Koopman, L.~Sitbon, and P.~Bruza, ``Medical semantic similarity with a neural language model,'' in \emph{Proceedings of the 23rd ACM International Conference on Conference on Information and Knowledge Management}, 2014, pp. 1819--1822.

\bibitem{yang2023one}
H.~Yang, M.~Li, Y.~Xiao, H.~Zhou, R.~Zhang, and Q.~Fang, ``One {LLM} is not enough: Harnessing the power of ensemble learning for medical question answering,'' \emph{medRxiv 2023.12.21.23300380}, 2023.

\bibitem{ye2023qilin}
Q.~Ye, J.~Liu, D.~Chong, P.~Zhou, Y.~Hua, and A.~Liu, ``{Qilin-Med}: Multi-stage knowledge injection advanced medical large language model,'' \emph{arXiv:2310.09089}, 2023.

\bibitem{gilbert2023large}
S.~Gilbert, H.~Harvey, T.~Melvin, E.~Vollebregt, and P.~Wicks, ``Large language model {AI} chatbots require approval as medical devices,'' \emph{Nature Medicine}, vol.~29, no.~10, pp. 2396--2398, 2023.

\bibitem{tang2023evaluating}
L.~Tang, Z.~Sun, B.~Idnay, J.~G. Nestor, A.~Soroush, P.~A. Elias, Z.~Xu, Y.~Ding, G.~Durrett, J.~F. Rousseau \emph{et~al.}, ``Evaluating large language models on medical evidence summarization,'' \emph{NPJ Digital Medicine}, vol.~6, no.~1, p. 158, 2023.

\bibitem{mesko2023impact}
B.~Mesk{\'o}, ``The impact of multimodal large language models on health care's future,'' \emph{Journal of Medical Internet Research}, vol.~25, p. e52865, 2023.

\bibitem{zarocostas2020fight}
J.~Zarocostas, ``How to fight an infodemic,'' \emph{The lancet}, vol. 395, no. 10225, p. 676, 2020.

\bibitem{liu2021framework}
T.~Liu and X.~Xiao, ``A framework of {AI}-based approaches to improving {eHealth} literacy and combating infodemic,'' \emph{Frontiers in Public Health}, vol.~9, p. 755808, 2021.

\bibitem{Gallotti_2020}
R.~Gallotti, F.~Valle, N.~Castaldo, P.~Sacco, and M.~De~Domenico, ``Assessing the risks of `infodemics' in response to {COVID-19} epidemics,'' \emph{Nat. Hum. Behav.}, vol.~4, pp. 1285--1293, 2020.

\bibitem{guo2022survey}
Z.~Guo, M.~Schlichtkrull, and A.~Vlachos, ``A survey on automated fact-checking,'' \emph{Transactions of the Association for Computational Linguistics}, vol.~10, pp. 178--206, 2022.

\bibitem{karadzhov2017fully}
G.~Karadzhov, P.~Nakov, L.~M{\`a}rquez, A.~Barr{\'o}n-Cede{\~n}o, and I.~Koychev, ``Fully automated fact checking using external sources,'' \emph{arXiv:1710.00341}, 2017.

\bibitem{zhou2019physiological}
J.~Zhou, H.~Hu, Z.~Li, K.~Yu, and F.~Chen, ``Physiological indicators for user trust in machine learning with influence enhanced fact-checking,'' in \emph{Machine Learning and Knowledge Extraction}, 2019, pp. 94--113.

\bibitem{nakashole2014language}
N.~Nakashole and T.~Mitchell, ``Language-aware truth assessment of fact candidates,'' in \emph{Proceedings of the 52nd Annual Meeting of the Association for Computational Linguistics (Volume 1: Long Papers)}, 2014, pp. 1009--1019.

\bibitem{harrag2022arabic}
F.~Harrag and M.~K. Djahli, ``Arabic fake news detection: A fact checking based deep learning approach,'' \emph{Transactions on Asian and Low-Resource Language Information Processing}, vol.~21, no.~4, pp. 1--34, 2022.

\bibitem{ciampaglia2015computational}
G.~L. Ciampaglia, P.~Shiralkar, L.~M. Rocha, J.~Bollen, F.~Menczer, and A.~Flammini, ``Computational fact checking from knowledge networks,'' \emph{PLoS ONE}, vol.~10, no.~6, p. e0128193, 2015.

\bibitem{gad2019tracy}
M.~H. Gad-Elrab, D.~Stepanova, J.~Urbani, and G.~Weikum, ``Tracy: Tracing facts over knowledge graphs and text,'' in \emph{The World Wide Web Conference}, 2019, pp. 3516--3520.

\bibitem{shi2016fact}
B.~Shi and T.~Weninger, ``Fact checking in heterogeneous information networks,'' in \emph{Proceedings of the 25th International Conference Companion on World Wide Web}, 2016, pp. 101--102.

\bibitem{vedula2021face}
N.~Vedula and S.~Parthasarathy, ``{FACE-KEG}: Fact checking explained using knowledge graphs,'' in \emph{Proceedings of the 14th ACM International Conference on Web Search and Data Mining}, 2021, pp. 526--534.

\bibitem{koloski2022knowledge}
B.~Koloski, T.~S. Perdih, M.~Robnik-{\v{S}}ikonja, S.~Pollak, and B.~{\v{S}}krlj, ``Knowledge graph informed fake news classification via heterogeneous representation ensembles,'' \emph{Neurocomputing}, vol. 496, pp. 208--226, 2022.

\bibitem{hu2021compare}
L.~Hu, T.~Yang, L.~Zhang, W.~Zhong, D.~Tang, C.~Shi, N.~Duan, and M.~Zhou, ``Compare to the knowledge: Graph neural fake news detection with external knowledge,'' in \emph{Proceedings of the 59th Annual Meeting of the Association for Computational Linguistics and the 11th International Joint Conference on Natural Language Processing (Volume 1: Long Papers)}, 2021, pp. 754--763.

\bibitem{mayank2022deap}
M.~Mayank, S.~Sharma, and R.~Sharma, ``{DEAP-FAKED}: Knowledge graph based approach for fake news detection,'' in \emph{2022 IEEE/ACM International Conference on Advances in Social Networks Analysis and Mining (ASONAM)}, 2022, pp. 47--51.

\bibitem{DING20071223}
J.~Ding and A.~Zhou, ``Eigenvalues of rank-one updated matrices with some applications,'' \emph{Applied Mathematics Letters}, vol.~20, no.~12, pp. 1223--1226, 2007.

\bibitem{MC_Levin}
D.~A. Levin, Y.~Peres, and E.~L. Wilmer, \emph{Markov Chains and Mixing Times}.\hskip 1em plus 0.5em minus 0.4em\relax Providence, Rhode Island: American Mathematical Society, 2017.

\bibitem{liu2024can}
Z.~Liu, X.~He, Y.~Tian, and N.~V. Chawla, ``Can we soft prompt {LLMs} for graph learning tasks?'' in \emph{Companion Proceedings of the ACM on Web Conference 2024}, 2024, pp. 481--484.

\bibitem{chen2024exploring}
Z.~Chen, H.~Mao, H.~Li, W.~Jin, H.~Wen, X.~Wei, S.~Wang, D.~Yin, W.~Fan, H.~Liu \emph{et~al.}, ``Exploring the potential of large language models ({LLMs}) in learning on graphs,'' \emph{ACM SIGKDD Explorations Newsletter}, vol.~25, no.~2, pp. 42--61, 2024.

\bibitem{auer2007dbpedia}
S.~Auer, C.~Bizer, G.~Kobilarov, J.~Lehmann, R.~Cyganiak, and Z.~Ives, ``{DBpedia}: A nucleus for a web of open data,'' in \emph{International Semantic Web Conference}, 2007, pp. 722--735.

\bibitem{reid2024gemini}
M.~Reid, N.~Savinov, D.~Teplyashin, D.~Lepikhin, T.~Lillicrap, J.-b. Alayrac, R.~Soricut, A.~Lazaridou, O.~Firat, J.~Schrittwieser \emph{et~al.}, ``Gemini 1.5: Unlocking multimodal understanding across millions of tokens of context,'' \emph{arXiv:2403.05530}, 2024.

\bibitem{sarmah2024hybridrag}
B.~Sarmah, D.~Mehta, B.~Hall, R.~Rao, S.~Patel, and S.~Pasquali, ``Hybridrag: Integrating knowledge graphs and vector retrieval augmented generation for efficient information extraction,'' in \emph{Proceedings of the 5th ACM International Conference on AI in Finance}, 2024, pp. 608--616.

\bibitem{guo2024lightrag}
Z.~Guo, L.~Xia, Y.~Yu, T.~Ao, and C.~Huang, ``Lightrag: Simple and fast retrieval-augmented generation,'' \emph{arXiv:2410.05779}, 2024.

\end{thebibliography}
 
%

\end{document}